\documentclass[nohyperref]{article}

\usepackage{microtype}
\usepackage{enumitem}
\usepackage{graphicx}
\usepackage{subfigure}
\usepackage{enumitem}
\usepackage{booktabs} 
\usepackage{tikz}
\usetikzlibrary{arrows}
\usepackage{hyperref}
\usepackage[textsize=tiny]{todonotes}
\usepackage{smile}
\usepackage{bbding}

\usepackage{colortbl}
\definecolor{LightCyan}{rgb}{0.8, 0.9, 1}

\usepackage{enumitem}
\usepackage{colortbl}
\definecolor{LightCyan}{rgb}{0.8, 0.9, 1}
\definecolor{LightGray}{rgb}{0.9,0.9,0.9}

\usepackage[accepted]{icml2023}

\usepackage[capitalize,noabbrev]{cleveref}

\allowdisplaybreaks

\icmltitlerunning{Optimal Horizon-Free Reward-Free Exploration for Linear Mixture MDPs}

\begin{document}

\twocolumn[
\icmltitle{Optimal Horizon-Free Reward-Free Exploration for Linear Mixture MDPs}



\icmlsetsymbol{equal}{*}

\begin{icmlauthorlist}
\icmlauthor{Junkai Zhang}{uclaml}
\icmlauthor{Weitong Zhang}{uclaml}
\icmlauthor{Quanquan Gu}{uclaml}
\end{icmlauthorlist}

\icmlaffiliation{uclaml}{Department of Computer Science, University of California, Los Angeles, California, USA}

\icmlcorrespondingauthor{Quanquan Gu}{qgu@cs.ucla.edu}

\icmlkeywords{Machine Learning, ICML}

\vskip 0.3in
]



\printAffiliationsAndNotice{} 
\begin{abstract}
We study reward-free reinforcement learning (RL) with linear function approximation, where the agent works in two phases: (1) in the exploration phase, the agent interacts with the environment but cannot access the reward; and (2) in the planning phase, the agent is given a reward function and is expected to find a near-optimal policy based on samples collected in the exploration phase. The sample complexities of existing reward-free algorithms have a polynomial dependence on the planning horizon, which makes them intractable for long planning horizon RL problems. 
In this paper, we propose a new reward-free algorithm for learning linear mixture Markov decision processes (MDPs), where the transition probability can be parameterized as a linear combination of known feature mappings. At the core of our algorithm is uncertainty-weighted value-targeted regression with exploration-driven pseudo-reward and a high-order moment estimator for the aleatoric and epistemic uncertainties. When the total reward is bounded by $1$, we show that our algorithm only needs to explore $\tilde O\left( d^2\varepsilon^{-2}\right)$ episodes to find an $\varepsilon$-optimal policy, where $d$ is the dimension of the feature mapping. The sample complexity of our algorithm only has a polylogarithmic dependence on the planning horizon and therefore is ``horizon-free''. 
In addition, we provide an $\Omega\left(d^2\varepsilon^{-2}\right)$ sample complexity lower bound, which matches the sample complexity of our algorithm up to logarithmic factors, suggesting that our algorithm is optimal.
\end{abstract}

\section{Introduction}
In Reinforcement Learning (RL), the agent sequentially interacts with the environment by executing policies and receiving observations, including states and rewards. The goal of the agent is to maximize the total reward. To achieve this goal, the agent needs to explore the environment and exploit the collected information to find the optimal policy. The exploration has long been considered as a central challenge for RL, for which the agent needs to strategically visit states to learn transition dynamics and the value of different states. RL algorithms are often designed to exploit the transition and the reward information  to achieve efficient exploration.

Unfortunately, in many real-world RL problems, reward functions are manually designed to incentive the agent to learn specific tasks, and they may change over time~\citep{achiam2017constrained, tessler2018reward, miryoosefi2019reinforcement}. To avoid learning the transition dynamics repeatedly, \citet{jin2020reward} proposed a new RL paradigm, \textit{Reward-free Exploration}, which separates exploration and planning into two different phases. In the exploration phase, the agent cannot access the real reward function. It can only learn the transition dynamics based on the collected episodes. While in the planning phase, the agent can no longer interact with the environment but has access to the reward function. The goal is to find the optimal policy based on the reward function and previous exploration. A series of work \citep{jin2020reward, kaufmann2021adaptive, menard2021fast, zhang2020nearly} have achieved the optimal sample complexity of $\tilde{O}(H^2S^2A\varepsilon^{-2})$, where $H$ is the planning horizon, $S$ is the number of states, and $A$ is the number of actions.

The sample complexity for learning tabular MDPs shows that learning becomes intractable when the sizes of the state and action spaces increase without further structural assumptions. Linear function approximation is a classical approach to deal with this challenge, which approximates the transition dynamic or the value function by linear functions on compact feature mappings. To this end, \citet{wang2020reward, zanette2020provably} studied reward-free RL in linear MDPs \citep{yang2019sample,jin2020provably}. \citet{zhang2021reward} studied reward-free exploration for linear mixture MDPs \citep{ayoub2020model,zhou2020provably}, where the transition probability is a linear combination of feature mappings. The subsequent work \citet{chen2022nearoptimal} has achieved near optimal sample complexity $\widetilde{O}\left(H^3d(H+d)\varepsilon^{-2}\right)$. However, this sample complexity depends on the planning horizon, which will blow up for long planning horizon problems. Therefore, a natural question arises:
\begin{center}\emph{
Can we design horizon-free, minimax optimal reward-free RL algorithms with
linear function approximation? }
\end{center}

In this paper, we answer this question affirmatively for linear mixture MDPs. In detail, our contributions are highlighted as follows.

\begin{itemize}
    \item We propose an algorithm for reward-free exploration in the linear mixture MDP setting. The algorithm guides the agent to collect samples using a well-designed exploration-driven pseudo-reward function. With a novel analysis based on high-order moment estimation that precisely controls the aleatoric and epistemic uncertainties, our algorithm can achieve an $\tilde{O}(d^2\varepsilon^{-2})$ sample complexity. This complexity only has polylogarithmic dependence on $H$.

    \item We show that any reward-free algorithm needs to explore at least $\Omega(d^2\varepsilon^{-2})$ episodes, to achieve an $\varepsilon$-optimal policy for any reward function, by constructing a special class of linear mixture MDPs. This lower bound matches the upper bound of our algorithm up to logarithmic factors, which indicates that our algorithm is optimal. 
      \item When rescaling the reward to satisfy $\sum_{h=1}^H r_h(s_h,a_h) \le H$, our algorithm achieves an $\tilde{O}(d^2H^2\varepsilon^{-2})$ sample complexity. This improves the sample complexity in the previous work \citet{chen2022nearoptimal}, which requires the $d>H$ condition to match the lower bound.
      
\end{itemize}
\paragraph{Notation}
We use the lowercase letter to denote scalars and lower and uppercase boldface letters to denote vectors and matrices, respectively. We denote by $[n]$ the set $\{1,\cdots, n\}$, and by $\overline{[n]}$ the set $\{0,\cdots,n-1\}$. For a vector $\xb$ and a positive semi-definite matrix $\bSigma$, we denote by $\| \xb \|_2$ the vector's Euclidean norm and define $\| \bx \|_{\bSigma} = \sqrt{\bx^\top\bSigma\bx}$. For two positive sequences $\{ 
 a_n\}$ and $\{ b_n \}$ with $n = 1,2,\cdots$, we write $a_n = O(b_n) $ if there exists an absolute constant $C>0$ such that $a_n \le C b_n$ holds for all $n\ge 1$, write $a_n = \Omega(b_n) $ if there exists an absolute constant $C>0$ such that $a_n \ge C b_n$ holds for all $n \ge 1$, and write $a_n = o(b_n)$ if $a_n/b_n \rightarrow 0$ as $n\rightarrow\infty$. We use $\tilde{O}(\cdot)$ and $\tilde{\Omega}(\cdot)$ to further hide the polylogarithmic factors.

\section{Related Work}
\paragraph{\textbf{RL} with Linear Function Approximation}
In recent years, a series of works have been devoted to the study of RL with linear function approximation \citep{jiang2017contextual, dann2018oracle, yang2019sample, wang2019optimism, du2019good, sun2019model, jin2020provably, zanette2020frequentist, zanette2020learning,  yang2020reinforcement, modi2020sample, ayoub2020model, jia2020model, cai2019provably, weisz2021exponential, zhou2021provably,  zhou2020nearly, he2022nearly, agarwal2022vo}.
Our work belongs to the linear mixture MDP setting \citep{yang2019sample, modi2020sample, ayoub2020model, jia2020model, zhou2020nearly, zhou2021provably},
where the transition kernel can be parameterized as a linear combination of some basic transition probability functions. \citet{zhou2020nearly} firstly achieved minimax regret $\tilde{O}\left(dH\sqrt{T}\right)$ in linear mixture MDPs by proposing a Bernstein-type concentration inequality for self-normalized martingales. Another kind of popular linearly parameterized MDP is linear MDP~\citep{wang2019optimism, du2019good, yang2020reinforcement, jin2020provably, zanette2020frequentist, wang2020reinforcement, he2021logarithmic}, which assumes both transition probability and reward function are linear functions of known feature mappings on state-action pairs. Under this setting, \citet{jin2020provably} firstly proposed statistically and computationally efficient algorithm LSVI-UCB and achieved $\tilde{O}\left( \sqrt{d^3H^3T}\right)$ regret bound. Recent works \citep{he2022nearly} further achieved nearly minimax optimal regret $\tilde{O}(d\sqrt{H^3K})$ by proposing computationally efficient algorithm LSVI-UCB++. Its concurrent work \citep{agarwal2022vo} achieves similar result under assumption $\sum_{h=1}^Hr_h(s_h,a_h)\le1$ with regret upper bound of $\tilde{O}(d\sqrt{HT}+d^6H^5)$.

\noindent\textbf{Reward-free \textbf{RL}}
Unlike standard RL settings, exploration and planning in reward-free RL are separated into two different phases. \citet{jin2020reward} achieved $\tilde{O}(H^5S^2A/\varepsilon^2)$ sample complexity in tabular MDPs by executing exploratory policy visiting states with probability proportional to its maximum visitation probability under any possible policy. Subsequent works~\citep{kaufmann2021adaptive, menard2021fast} proposed algorithms RF-UCRL and RF-Express to gradually improve the result to  $\tilde{O}\left(H^3S^2A\varepsilon^{-2}\right)$. The optimal sample complexity bound $\widetilde{O}(H^2S^2A\varepsilon^{-2})$ was achieved by algorithm SSTP proposed in \citet{zhang2020nearly}, which matched the lower bound provided in \citet{jin2020reward} up to logarithmic factors.

Recent years have witnessed a trend of reward-free exploration in RL with linear function approximation~\citep{wang2020reward, zanette2020provably, zhang2021reward, chen2022nearoptimal, huang2022towards, wagenmaker2022reward}. The near minimax optimal sample complexity of reward-free exploration in linear mixture MDP was achieved by \citet{chen2022nearoptimal} when $d>H$ using the well-designed exploration-driven pseudo reward function. On the other hand, in the linear MDP setting, \citet{wang2020reward} proposed the first efficient algorithm, which only required $\tilde{O}(H^6d^3\varepsilon^{-2})$ sample complexity. The subsequent works, \citet{chen2022nearoptimal} and \citet{wagenmaker2022reward}, gave sample complexity of $\tilde{O}(H^4d^3\varepsilon^{-2})$ and $\tilde{O}(H^5d^2\varepsilon^{-2})$, which are the sharpest for $H$ and $d$, respectively. Some significant works are summarized in Table~\ref{tab:rw}.

\newcolumntype{g}{>{\columncolor{LightCyan}}c}
\begin{table*}[ht]
\centering
\caption{Comparison of episodic reward-free RL algorithms. Column \textbf{Time Homo.} means if the algorithm is time-homogeneous ($\checkmark$) or not ($\times$), rows with light blue background indicates our results.}
\small{
\begin{tabular}{cgggg}
\hline
\rowcolor{white}
Setting & Algorithm & Rewards Scale & Time Homo. & Sample Complexity \\
\hline
\rowcolor{white}
 & \citet{jin2020reward} &$r_h(s_h,a_h)\in [0,1]$ & $\times$ & $\widetilde{O}(H^5S^2A\varepsilon^{-2})$ \\
\rowcolor{LightGray}
 \cellcolor{white} & \Gape[0pt][2pt]{\makecell[c]{RF-UCRL \\ \citep{kaufmann2020adaptive}}} &$r_h(s_h,a_h)\in [0,1]$& $\times$ & $\widetilde{O}(H^4S^2A\varepsilon^{-2})$ \\
 \rowcolor{white}
 \makecell[c]{Tabular \\ MDP} & \makecell[c]{RF-Express \\ \citep{menard2021fast}} &$r_h(s_h,a_h)\in [0,1]$&$\times$ &$\widetilde{O}(H^3S^2A\varepsilon^{-2})$ \\
 \rowcolor{LightGray}
 \cellcolor{white} & \Gape[0pt][2pt]{\makecell[c]{SSTP \\ \citep{zhang2020nearly}}} &$\sum_{h=1}^Hr_h(s_h,a_h)\le 1$& \checkmark &$\widetilde{O}(S^2A\varepsilon^{-2})$ \\
 \rowcolor{white}
 & \makecell[c]{Lower bound \\ \citep{jin2020reward}} &$r_h(s_h,a_h)\in [0,1]$& $\times$ &$\Omega(H^2S^2A\varepsilon^{-2})$ \\
 \rowcolor{LightGray}
 \cellcolor{white} & \Gape[0pt][2pt]{\makecell[c]{Lower bound \\ \citep{zhang2020nearly}} }&$\sum_{h=1}^Hr_h(s_h,a_h)\le 1$& \checkmark &$\Omega(S^2A\varepsilon^{-2})$ \\
\hline 
\rowcolor{white}
 & \citet{wang2020reward}  &$r_h(s_h,a_h)\in [0,1]$& $\times$&$\widetilde{O}\left(H^6 d^3 \varepsilon^{-2}\right)$ \\
\rowcolor{LightGray}
 \cellcolor{white}\makecell[c]{Linear \\ MDP} & \Gape[0pt][2pt]{\makecell[c]{FRANCIS \\ \citep{zanette2020provably}}}  &$r_h(s_h,a_h)\in [0,1]$& $\times$&$\widetilde{O}\left(H^5 d^3 \varepsilon^{-2}\right)$ \\
 \rowcolor{white}
 & \makecell[c]{RFLIN \\ \citep{wagenmaker2022reward}}  &$r_h(s_h,a_h)\in [0,1]$& $\times$ &$\widetilde{O}\left(H^5 d^2 \varepsilon^{-2}\right)$ \\
\hline
\rowcolor{LightGray}
 \cellcolor{white}& \Gape[0pt][2pt]{\makecell[c]{UCRL-RFE+ \\ \citep{zhang2021reward}} }&$r_h(s_h,a_h)\in [0,1]$& \checkmark &$\widetilde{O}\left(H^4d(H+d)\varepsilon^{-2}\right)$ \\
\rowcolor{white}
 Linear & \citet{chen2022nearoptimal} &$r_h(s_h,a_h)\in [0,1]$& $\times$ &  $\widetilde{O}\left(H^3d(H+d)\varepsilon^{-2}\right)$ \\
 Mixture & \textbf{Our work} (Cor.~\ref{cor:samplecomplexity}) &$\sum_{h=1}^Hr_h(s_h,a_h)\le 1$& \checkmark &$\widetilde{O}(d^2\varepsilon^{-2})$ \\
 MDP & \textbf{Our work}  (Cor.~\ref{cor:samplecomplexityH}) &$\sum_{h=1}^Hr_h(s_h,a_h)\le H$& \checkmark &$\tilde{O}(H^2d^2\varepsilon^{-2})$ \\
 & Lower bound  (Thm.~\ref{thm:lowerbound}) &$\sum_{h=1}^Hr_h(s_h,a_h)\le 1$&\checkmark &$\Omega\left(d^2  \varepsilon^{-2}\right)$ \\
 & Lower bound  (Cor.~\ref{cor:lowerboundH}) &$r_h(s_h,a_h)\in [0,1]$& \checkmark &$\Omega(H^2d^2\varepsilon^{-2})$ \\
\hline
\end{tabular}
}
\label{tab:rw}
\end{table*}

\noindent\textbf{Horizon-free \textbf{RL}}
The long planning horizon has long been viewed as RL's main challenge. However, a series of works has shown that RL is no more difficult than contextual bandits by removing the influence of the total reward scale. In tabular MDPs, the algorithm proposed in \citet{wang2020long} firstly achieved polylogarithmic $H$ dependency sample complexity bound $\tilde{O}(S^5A^4\varepsilon^{-2})$ by carefully reusing samples and avoid unnecessary sampling. \citet{zhang2021reinforcement} further proposed an improved algorithm \texttt{MVP} to achieve near-optimal regret bound $\tilde{O}(\sqrt{SAK}+S^2A)$ based on new Bernstein-type bonus. Similar polylogarithmic $H$ dependency bounds had been established by \citet{ren2021nearly} for linear MDP with anchor points, \citet{tarbouriech2021stochastic} for the stochastic shortest path. \citet{li2022settling} achieved the surprising $H$ independent sample complexity bound $O((SA)^{O(S)}\varepsilon^{-5})$ by building a connection between discounted MDPs and episodic MDPs and a novel perturbation analysis in MDPs. The algorithm proposed by \citet{zhang2022horizon} further improved the sample complexity to $O(S^9A^3\varepsilon^{-2}\text{polylog}(S,A,\varepsilon^{-1}))$ only depending on state and action spaces size polynomially by exploiting the power of stationary policy.  Thanks to the linear function approximation, \citet{zhou22high} firstly achieve horizon-free regret bound $\tilde{O}(d\sqrt{K} + d^2)$ independent of state and action spaces size. However, all the above works are limited to standard RL settings. In the paradigm of reward-free exploration, the only horizon-free result was achieved by \citet{zhang2021reward} with sample complexity bound of $\tilde{O}(S^2A\varepsilon^{-2})$, where the polynomial dependency on $S$ and $A$ is still unacceptable when the states and actions spaces are large. Our algorithm HF-UCRL-RFE++ establishes the first horizon-free sample complexity bound independent of state and action spaces size in reward-free exploration.

\section{Preliminaries}
We consider an episodic finite horizon Markov Decision Process (MDP) $\cM = \left(\cS,\cA, H,\{r_h\}_{h=1}^H, \PP\right)$, where $\cS$ is the countable (and maybe infinite) state space, $\cA$ is the action space, $H$ is the length of the episode, $r_h:\cS\times\cA \rightarrow [0,1]$ is the deterministic reward function, and $\PP: \cS\times\cA \rightarrow [0,1]$ is the time-homogeneous transition probability. 

Based on the current state $s\in\cS$ and the time step $h\in [H]$, a policy $\pi$ chooses the action $a\in\cA$ to be executed by the agent. Formally, we denote a policy as $\pi = \{ \pi_h \}_{h=1}^H$, where $\pi_h:\cS\rightarrow\cA$ is a function which maps a state $s$ to an action $a$. For any $(s,a) \in \cS\times\cA$ and $h\in[H]$, we define the action-value function $Q_h^\pi(s,a)$ and value function $V_h^\pi(s)$ as follows:
\begin{align*}
& Q_h^\pi(s,a) = \EE \bigg[ \sum_{h'=h}^H r(s_{h'},a_{h'}) \Big| s_h = s, a_h = a, \\
  & \qquad \qquad s_{h'}\sim\PP(\cdot|s_{h'-1},a_{h'-1}),a_{h'} = \pi_{h'}(s_{h'})\bigg], \\
& V_h^\pi (s) = Q_h^\pi\big(s,\pi_h(s)\big).
\end{align*}
We define the optimal value function $V_h^*(\cdot)$ and optimal action-value function $Q_h^*(\cdot,\cdot)$ as $V_h^*(\cdot) = \sup_\pi V_h^\pi(\cdot)$ and $Q_h^*(\cdot,\cdot) = \sup_{\pi}Q_h^\pi(\cdot,\cdot)$, respectively. For any function $V:\cS \rightarrow [0,1]$, we introduce the following short-hands to denote the conditional expectation and variance of $V$:
\begin{align*}
\left[\PP V\right](s,a) & = \EE_{s'\sim\PP(\cdot|s,a)}V\left(s'\right), \\
\left[\VV V\right](s,a) & = \left[\PP V^2\right](s,a) - \left[\PP V\right](s,a)^2.
\end{align*}
For any $(s, a) \in \cS\times\cA$ and $h\in[H]$, the Bellman equation and Bellman optimality equation can be defined recursively as
\begin{align*}
Q_h^\pi (s,a) &= r(s,a) + \left[\PP V_{h+1}^\pi\right](s,a), \\
Q_h^* (s,a) &= r(s,a) + \left[\PP V_{h+1}^*\right](s,a).
\end{align*}
In this paper, we make the structural assumption that the transition dynamic has a linear structure, which has been considered in prior works as below:

\begin{definition}[Linear Mixture MDPs, \citealt{jia2020model,ayoub2020model,zhou2020provably}]\label{def:mdp}
The unknown transition probability $\PP$ is a linear combination of $d$ signed basis measures $\bphi_i(s'|s,a)$, i.e., $\PP(s'|s,a) = \sum_{i=1}^d \bphi_i(s'|s,a)\btheta^*_i$. Meanwhile, for any $V: \cS \rightarrow [0,1]$, $i \in [d], (s,a) \in \cS \times \cA$, the summation $\sum_{s' \in \cS} \bphi_i(s'|s,a) V(s')$ can be calculated efficiently within $\cO$ time. For simplicity, let $\bphi = [\bphi_1, \dots, \bphi_d]^\top$, $\btheta^* = [\btheta^*_1,\dots, \btheta^*_d]^\top$ and $\bphi_V(s, a) = \sum_{s' \in \cS} \bphi(s'| s, a)V(s')$. W.l.o.g., we assume $\|\btheta^*\|_2 \le B, \|\bphi_V(s, a)\|_2 \le 1$ for all $V: \cS \rightarrow [0,1]$ and $(s,a) \in \cS \times \cA$.
\end{definition}

We further assume that the accumulated reward of an episode for any trajectory is upper bounded by 1, which ensures that the only factors affecting the final statistical complexity are difficulties brought by exploration and long planning horizon rather than the scale of the total reward.

\begin{assumption}\label{as:reward}
    (Bounded total reward) For any trajectory $\{s_h,a_h\}_{h=1}^H$, we have $0\leq\sum_{h=1}^H r_h(s_h,a_h) 
 \le 1$.
\end{assumption}
We denote the set of reward functions satisfying Assumption \ref{as:reward} by $\cR$.

\paragraph{Reward-free RL} In reward-free RL, the real reward function is accessible only after the agent finishes the interactions with the environment. Specifically, the algorithm can be separated into two phases: (i) \textit{Exploration phase}: the algorithm can't access the reward function but collect $K$ episodes of samples by interacting with the environment.  (ii) \textit{Planning phase}: The algorithm is given reward function $\{r_h\}_{h=1}^H$ and is expected to find the optimal policy. $(\varepsilon,\delta)$-learn and sample complexity of the algorithm is formally defined as follows.

\begin{definition}
 $((\varepsilon, \delta)$-learnability). Given an MDP transition kernel set $\mathcal{P}$, reward function set $\mathcal{R}$ and a initial state distribution $\mu$, we say a reward-free algorithm can $(\varepsilon, \delta)$-learn the problem $(\mathcal{P}, \mathcal{R})$ with sample complexity $K(\varepsilon, \delta)$, if for any transition kernel $P \in \mathcal{P}$, after receiving $K(\varepsilon, \delta)$ episodes in the exploration phase, for any reward function $r \in \mathcal{R}$, the algorithm returns a policy $\pi$ in planning phase, such that with probability at least $1-\delta,~V_1^*\left(s_1 ; r\right)-V_1^\pi\left(s_1 ; r\right) \leq \varepsilon$.
\end{definition}

\section{Algorithms}\label{sec:alg}

In this section, we propose our reward-free exploration algorithm HF-UCRL-RFE++. This algorithm consists of two phases. In the exploration phase, it builds an estimator $\btheta$ for the linear mixture MDP transition kernel parameter $\btheta^*$ based on the sampled episodes. At a high level, the estimation follows the \textit{value-targeted regression} (VTR) framework proposed by \citet{jia2020model}. The VTR is basically a ridge regression with value functions as responses and feature mappings as predictors. However, value functions have no estimates since the reward function is not accessible. Therefore, the value functions and reward functions are replaced by well-designed exploration-driven pseudo-value functions and pseudo-reward functions. To achieve a better estimation, we further apply the \textit{high-order moment estimation} (HOME) technique proposed by \citet{zhou22high}. Then, during the planning phase, the algorithm uses the estimator $\btheta$ acquired in the exploration phase to find the optimal policy $\pi$ for the given reward functions. Our algorithm is described in Algorithm \ref{alg:exp}.

\begin{algorithm*}[!ht]
\caption{HF-UCRL-RFE++}
\label{alg:exp}
\begin{algorithmic}[1]
\REQUIRE Confidence radius $\{\beta_k\}$, regularization parameter $\lambda$, number of the high-order estimator $M$.
\STATE \textbf{Phase I: Exploration Phase}
\STATE Initialize $\hat\bSigma_{1, 1, m} \leftarrow \lambda \Ib$, $\tilde \bSigma_{1, 1, m} \leftarrow \lambda \Ib$, $\tilde \bbb_{1, 1, m} \leftarrow \zero$,  $\hat\bbb_{1, 1, m} \leftarrow \zero$ for all $m \in \overline{[M]}$, $\cU_{1}=\left\{\btheta|\btheta\in\RR^d\right\}$.
\STATE Set $\hat \btheta_{1, m} \leftarrow \hat \bSigma_{1, 1, m}^{-1} \hat \bbb_{1, 1, m}$, $\tilde \btheta_{1, m} \leftarrow \tilde \bSigma_{1, 1, m}^{-1} \tilde \bbb_{1, 1, m}$ for all $m \in \overline{[M]}$.
\label{ln:sol}
\FOR {$k = 1, 2, \cdots, K$}
\STATE Calculate $ \pi_k, \btheta_k, r_k = \argmax_{\pi, \btheta \in \cU_{k},r\in R}\hat V_{k,1}(s_1; \btheta, \pi, r)$, where $\hat V_{k,1}$ is defined in \eqref{eq:pseudo_V}. Denote $\big\{\tilde V_{k,h}(\cdot)\big\}_{h=1}^H = \left\{V_h(\cdot;\btheta_k,\pi_k,r_k)\right\}_{h=1}^H$. \label{ln:optimization}
\STATE Receive initial state $s_1^k = s_1$. \label{ln:trajectory_start}
\FOR {$h = 1, 2, \cdots, H$}
\STATE Execute $a_h^k = \pi_h^k\left(s_h^k\right)$, receive $s_{h+1}^k \sim \PP\left(\cdot| s_h^k, a_h^k\right)$.
\STATE For $m \in \overline{[M]}$, denote $\hat \bphi_{k, h, m} = \bphi_{\hat V_{k,h+1}^{2^m}}(s_h^k, a_h^k)$, $\tilde \bphi_{k, h, m} = \bphi_{\tilde V_{k,h+1}^{2^m}}(s_h^k, a_h^k)$.
\STATE Set $\big\{\hat \sigma_{k, h, m}\big\} \leftarrow \text{HOME}_{\text{Alg.~\ref{alg:home}}}\Big(\big\{\hat\bphi_{k, h, m}, \hat \btheta_{k, m}, \hat\bSigma_{k, h, m}, \dot{\hat\bSigma}_{k, m}\big\}, \beta_k, \alpha, \gamma\Big)$.\label{ln:home}
\STATE Set $\big\{\tilde \sigma_{k, h, m}\big\} \leftarrow \text{HOME}_{\text{Alg.~\ref{alg:home}}}\Big(\big\{\tilde \bphi_{k, h, m}, \tilde \btheta_{k, m}, \tilde\bSigma_{k, h, m}, \dot{\tilde\bSigma}_{k, m}\big\}, \beta_k, \alpha, \gamma\Big)$. \label{ln:home2}
\STATE Set $\tilde \bSigma_{k, h + 1, m} \leftarrow \tilde \bSigma_{k, h, m} + \tilde \bphi_{k, h, m}\tilde \bphi_{k, h, m}^\top\tilde \sigma_{k, h, m}^{-2}$ for $m \in \overline{[M]}$. \label{ln:wr1}
\STATE Set $\hat \bSigma_{k, h + 1, m} \leftarrow \hat \bSigma_{k, h, m} + \hat \bphi_{k, h, m}\hat \bphi_{k, h, m}^\top\hat \sigma_{k, h, m}^{-2}$ for $m \in \overline{[M]}$.
\STATE Set $\tilde\bbb_{k, h + 1, m} \leftarrow \tilde\bbb_{k, h, m} + \tilde\bphi_{k, h, m}\tilde V_{k, h+1}^{2^m}(s_{h+1}^k)\tilde \sigma_{k, h, m}^{-2}$ for $m \in \overline{[M]}$. \label{ln:wr2}
\STATE Set $\hat \bbb_{k, h + 1, m} \leftarrow \hat \bbb_{k, h, m} + \hat \bphi_{k, h, m}\hat V_{k, h+1}^{2^m}(s_{h+1}^k)\hat \sigma_{k, h, m}^{-2}$ for $m \in \overline{[M]}$.
\ENDFOR
\STATE $\dot{\tilde{\bSigma}}_{k+1, m} \leftarrow \tilde \bSigma_{k , H + 1, m}$, $\dot{\hat{\bSigma}}_{k+1, m} \leftarrow \hat \bSigma_{k , H + 1, m}$. \label{ln:dot}
\STATE Set $\tilde \bSigma_{k+1, 1, m} \leftarrow \tilde \bSigma_{k , H + 1, m}, \tilde \bbb_{k+1, 1, m} \leftarrow \tilde\bbb_{k, H + 1, m}$, $\tilde \btheta_{k+1,m} = \tilde\bSigma_{k+1, 1, m}^{-1}\tilde\bbb_{k+1, 1, m}$.
\STATE Set $\hat\bSigma_{k+1, 1, m} \leftarrow \hat\bSigma_{k , H + 1, m}, \hat\bbb_{k+1, 1, m} \leftarrow \hat\bbb_{k, H + 1, m}$, $\hat \btheta_{k+1,m} = \hat\bSigma_{k+1, 1, m}^{-1}\hat\bbb_{k+1, 1, m}$.
\STATE Update the confidence set $\cU_{k}$ to $\cU_{k+1}$ by adding constraints 
\eqref{up:hat},~\eqref{up:tilde}. \label{ln:update_confidence_set}
\label{ln:confidence_set}
\ENDFOR \label{ln:trajectory_end}
\STATE \textbf{Phase II: Planning Phase}
\STATE Receive reward function $r$. \label{ln:receive_reward}
\STATE $\hat\pi_r = \arg\max_{\pi} V_1(\cdot;\btheta_K, \pi, r)$. \label{ln:optimistic_planning}
\STATE Return policy $\hat\pi_r$. \label{ln:return_policy}
\end{algorithmic}
\end{algorithm*}

\paragraph{Exploration-driven Pseudo Value Function}
As mentioned above, in the paradigm of reward-free exploration, we have to construct the pseudo-reward function to guide the agent in taking actions in the absence of the real reward function. As we adopt in this work, the most natural idea is to construct the pseudo-reward function related to uncertainty, which urges the agent to collect information about the most uncertain states and actions. Two approaches follow this idea: one is constructing the pseudo reward function directly measuring and maximizing the uncertainty of each stage, and the other is constructing the pseudo reward function maximizing the overall uncertainty along trajectories. \citet{zhang2021model} took the first approach, constructing the pseudo-reward function in the form of
\begin{align}
r^k_{h}(s,a) & = \min\Big\{1, \frac{2\beta}{H}\sqrt{\max_{V\in \cS\rightarrow [0,H-h]} \| \bphi_V (s,a) \|_{\bSigma_{1,k}^{-1}}} \Big\},    \notag 
\end{align}
and the pseudo-value function to be the argument of the maxima for the above uncertainty measure. Under this construction, the suboptimality in the planning phase can be bounded by the accumulation of uncertainty. This approach is straightforward but has the following two drawbacks. Firstly, without the truncation for accumulation of uncertainty, the upper bound of overall suboptimality in the planning phase will be in the scale of $O(H)$, which is meaningless since the value function lies in the interval of $[0,1]$ under our assumption. {Second, since VTR utilizes value functions' variance information for $\btheta$ estimation, it requires a Bellman-equation-type equality between two consecutive stages $h$ and $h+1$. However, the first approach does not satisfy this requirement, preventing us from acquiring a more accurate estimate.}

To address the above issues, we follow the design of pseudo value function proposed in \citet{chen2022nearoptimal}. In particular, we are constructing the pseudo-reward function aiming to maximize the overall uncertainty along trajectories. We view the uncertainty of states and actions as a function of (pseudo) reward function $r$, policy $\pi$, and transition kernel parameter $\btheta$ defined as follows
\begin{align}
u_{k,h}(s,a; \btheta,\pi, r)  =  &\min\Big\{1, \nonumber \\
     &\beta\big\| \bphi_{V_h(\cdot;\btheta,\pi,r)}(s,a)\big\|_{\dot{\tilde\bSigma}_{k,0}^{-1}}\Big\}, \label{eq:pseudo_r}
\end{align}
where $V_h(\cdot;\btheta,\pi,r)$ is the the value function of policy $\pi$ for linear mixture MDP with transition kernel parameter $\btheta$ and the reward function $r$, and the overall uncertainty along the trajectory is the truncated sum of each step uncertainty defined as
\begin{align}
    \bar V_{k,h}(s;\btheta,\pi,r)  = \min \Big\{1, & u_{k,h}(s,\pi(s);\btheta,\pi, r) \notag
      \\ 
      & + \bphi^\top_{\bar V_{k,h+1}(\cdot;\btheta,\pi,r)} (s,\pi(s))\btheta^*\Big\}. \notag
\end{align}
However, the definition of $ \bar V_{k,h}(s;\btheta,\pi,r)$ involves $\btheta^*$ , which is unknown to the agent. Hence, we construct the optimistic estimation of $\bar V_{k,h}(s;\btheta,\pi,r)$ as $\hat V_{k,h}(s;\btheta,\pi,r)$ defined as
\begin{align}
\hat V_{k,h}(s;\btheta,\pi,&r)  = \min \Big\{1, u_{k,h}(s,\pi(s);\btheta,\pi, r) \nonumber  \\
    \quad & +  2 \beta  \big\| \bphi_{\hat V_{k,h+1}(\cdot;\btheta,\pi,r)} 
    (s,\pi(s)\big)\big\|_{\dot{\hat\bSigma}_{k,0}^{-1}} \nonumber 
       \\
      & + \bphi^\top_{\hat V_{k,h+1}(\cdot;\btheta,\pi,r)} (s,\pi(s))\btheta\Big\}.    \label{eq:pseudo_V}
\end{align}
Notable, the definitions of $u_{k,h}$ and $\hat{V}_{k,h}$ involve the covariance matrices $\dot{\tilde\bSigma}_{k,0}$ and $\dot{\hat\bSigma}_{k,0}$, which are computed at the end of the preceding episode at Line~\ref{ln:dot} of Algorithm~\ref{alg:exp}. In the following content, when there is no confusion, we may write $ \hat V_{k,h}(\cdot) = \hat V_{k,h}(\cdot;\btheta_k,\pi_k, r_k)$, $u_{k,h}(\cdot,\cdot) = u_{k,h}(\cdot,\cdot;\btheta_k,\pi_k, r_k)$. In order to collect more information, the agent is expected to transit through the trajectory with the largest uncertainty $\hat V_{k,h}$. It is notable that $\hat V_{k,h}$ is a function of (pseudo) reward function $r_k$, policy $\pi_k$, and transition kernel parameter $\btheta_k$. Thus, at the beginning of each episode, we set $r_k$, $\pi_k$, and $\btheta_k$ to be arguments of the maxima, as presented in Line~\ref{ln:optimization} in Algorithm~\ref{alg:exp}. Through this process, we acquire the pseudo value function $r_k$, which is essential for reward-free exploration. Afterward, the algorithm collects samples along trajectories induced by policy $\pi_k$ and improves the estimation of $\btheta_k$ in Line~\ref{ln:trajectory_start} to Line~\ref{ln:trajectory_end}. In this stage, Algorithm~\ref{alg:exp} encounters two series of functions in the form of Bellman equations; one is the sum of pseudo rewards $r$, $\tilde{V}_{k,h}(\cdot) = V_{h}(\cdot;\btheta_k,\pi_k, r_k)$, which we refer as pseudo value function, and one is the uncertainty along the trajectory, $\hat{V}_{k,h}$. These two series of functions are both eligible for refined VTR and thus help estimate $\btheta$, as we will explain in the following.

\begin{algorithm*}[!ht]
\caption{High-order moment estimator (HOME)}
\label{alg:home}
\begin{algorithmic}[1]
\REQUIRE Features $\left\{\bphi_{k,h,m}\right\}_{m \in \overline{[M]}}$, vector estimators $\left\{\btheta_{k,m}\right\}_{m \in \overline{[M]}}$, covariance matrix $\left\{\bSigma_{k,h,m}, \dot{\bSigma}_{k, m}\right\}_{m \in \overline{[M]}}$, confidence radius $\beta_k$, $\alpha, \gamma$.
\FOR{$m = 0,\dots, M-2$}
\STATE Set $\left[\overline\VV_{k,m}V_{k,h+1}^{2^m}\right](s_h^k, a_h^k) \leftarrow  \left[\bphi_{k,h,m+1}^\top \btheta_{k,m+1}\right]_{[0, 1]} -  \left[\bphi_{k,h,m}^\top\btheta_{k,m}\right]_{[0,1]}^2$.
\STATE Set $E_{k,h,m} \leftarrow   \left[2\beta_k\left\|\bphi_{k,h,m}\right\|_{\dot{\bSigma}_{k, m}^{-1}}\right]_{[0, 1]} + \left[\beta_k\left\|\bphi_{k,h,m+1}\right\|_{\dot{\bSigma}_{k,m+1}^{-1}}\right]_{[0, 1]}$.
\STATE Set $\overline\sigma_{k,h,m}^2\leftarrow \max\left\{ \gamma^2\left\|\bphi_{k,h,m}\right\|_{\bSigma_{k,h,m}^{-1}}, \left[\overline\VV_{k,m}V_{k, h+1}^{2^m}\right](s_h^k, a_h^k) + E_{k,h,m}, \alpha^2\right\}$. \label{ln:variance}
\ENDFOR
\STATE  Set $\overline\sigma_{k,h,M-1}^2\leftarrow\max\left\{ \gamma^2\left\|\bphi_{k,h,M-1}\right\|_{\bSigma_{k,h,M-1}^{-1}}, 1,   \alpha^2 \right\}$.
\ENSURE $\left\{\overline\sigma_{k,h,m}\right\}_{m \in \overline{[M]}}$.
\end{algorithmic}
\end{algorithm*}

\paragraph{High-order Moment Estimation}
The key technique used in our algorithm consists of two series of high-order estimations for the transition kernel parameter $\btheta$. The algorithm for high-order moment estimation is stated in Algorithm~\ref{alg:home}. In the exploration phase, the agent learns the environment with the help of two series of value functions $\tilde{V}_{k,h}$ and $\hat{V}_{k,h}$. They serve to characterize different aspects of the model, one for pseudo values and one for trajectory uncertainty. And thus, they rely on different estimations of transition kernel parameter $\btheta$. Two independent series of higher-order moment estimations are necessary for achieving accurate estimation. In the Algorithm~\ref{alg:exp}, both estimations of $\btheta$ are the solutions to the weighted regression problem in the following form:
\begin{align}
&\argmin_{\btheta}\bigg(  \lambda \|\btheta\|_2^2 \nonumber \\ 
& \quad + \sum_{j=1}^{k-1}\sum_{h=1}^{H}\big( \bphi_{j,h,0}^\top \btheta - V_{j,h}(s^j_{h+1})\big)^2 /\bar{\sigma}_{j,h,0}^2\bigg), \label{eq:reg}
\end{align}
where the regression weight $\bar{\sigma}_{j,h,0}$ is set as Equation~\eqref{eq:weight}.
\begin{align}
 \overline\sigma_{k,h,0}^2\leftarrow  &\max\Big\{ \gamma^2\left\|\bphi_{k,h,0}\right\|_{\bSigma_{k,h,0}^{-1}},  \notag
  \\ &\left[\overline\VV_{k,0}V_{k, h+1}\right](s_h^k, a_h^k) + E_{k,h,0}, \alpha^2 \Big\}. \label{eq:weight}
\end{align}
$\bar{\sigma}_{j,h,0}$ can be considered as an combination of \textit{aleatoric uncertainty} and \textit{epistemic uncertainty} \citep{kendall2017uncertainties, mai2022sample}. The first term $\gamma^2\left\|\bphi_{k,h,m}\right\|_{\bSigma_{k,h,0}^{-1}}$ in~\eqref{eq:weight} is the \textit{epistemic uncertainty} caused by limited available data. And the second term in Equation~\eqref{eq:weight} is supposed to be the \textit{aleatoric uncertainty} $ \VV_{k,0}V_{k, h+1} $ characterizing the inherent non-determinism of the transition kernel, which is irreducible. Here the $\VV_{k,m}V_{k, h+1}$ is the variance of $V_{k, h+1}$ to $2^m$ defined as $[\PP V_{k,h+1}^{2^{m+1}}](s_h^k,a_h^k) -  [\PP V_{k,h+1}^{2^{m}}](s_h^k,a_h^k)^2$. Then, $\big[\VV_{k,0}V_{k, h+1}\big] (s_h^k, a_h^k)$ is further replaced with its estimate $\big[\overline\VV_{k,0}V_{k, h+1}\big] (s_h^k, a_h^k)$ plus its error bound $E_{k,h,0}$ since real variance $\big[\VV_{k,0}V_{k, h+1}\big] (s_h^k, a_h^k)$ is unknown to the agent. Because $\big[\VV_{k,0}V_{k, h+1}\big] (s_h^k, a_h^k)$ is a quadratic function of the real transition kernel parameter $\btheta^*$, its estimate can be achieved as 
\begin{align}
\left[\overline{\mathbb{V}}_{k, 0}  V_{k, h+1}\right]\left(s_h^k, a_h^k\right)=&\Big[\Big\langle\bphi_{k, h, 1}, \btheta_{k, 1}\Big\rangle\Big]_{[0,1]} \notag \\
&-\left[\left\langle\hat\bphi_{k, h, 0}, \btheta_{k, 0}\right\rangle\right]_{[0,1]}^2,\notag
\end{align}
where $\btheta_{k,1}$ is again the solution to the weighted regression problem similar to \eqref{eq:weight} with predictors $\bphi_{k,h,1} = \bphi_{V^2_{k,h+1}}(s_h^k,a_h^k)$, responses $V^2_{k,h+1}(s_{h+1}^k)$ and weight $\overline\sigma_{k,h,1}$. Following the above idea, the value of weight $\overline\sigma_{k,h,1}$ further relies on $\btheta_{k,2}$, which is the solution to a weighted regression problem involving another weight $\overline\sigma_{k,h,2}$. The algorithm carried out this process recursively until $\overline\sigma_{k,h,M-1}$, where its second term is replaced by the trivial upper bound of aleatoric uncertainty.

Applying HOME to the reward-free setting brings additional difficulties in controlling the error of our estimate for the model, as the error introduced by using the pseudo reward function instead of the real reward function and the error introduced by estimating the true transition kernel must be controlled separately. To address this problem, we carefully estimate variables indicating different kinds of error into two series of HOME in Line~\ref{ln:home} and Line~\ref{ln:home2}. Since the separation of variables deeply exploits the inner structure of the problem, the two series of HOME can be merged in the end to achieve a unified control for both kinds of error.

Previous work \citet{chen2022nearoptimal} implemented the weighted value regression in a more crude way. The weights are constructed only on aleatoric uncertainty, totally ignoring epistemic uncertainty. In addition, they use the same instead of different transition kernel parameters to calculate different order moments of the value function and stop target value regression at second order moment, which increased avoidable error. As a result, \citet{chen2022nearoptimal} can only replace factor $Hd$ with factor $H+d$ when trying to improve the dependency on $d$ in the upper bound. In contrast, our work further improves factor $H+d$ to factor $H$ through the well-designed target value regression, as we can see in Corollary~\ref{cor:samplecomplexityH}.

\paragraph{High Confidence Set}
At the end of each episode, we add the following constraints into $\cU_{k}$ to update the high confidence set in Line~\ref{ln:update_confidence_set} of Algorithm~\ref{alg:exp}.
\begin{align} 
    \left\| \btheta  - \hat \btheta_{k,m} \right\|_{\dot{\hat\bSigma}_{k,m}} & \le \beta_k,~m \in \overline{[M]}, \label{up:hat} \\
    \left\| \btheta  - \tilde \btheta_{k,m} \right\|_{\dot{\tilde\bSigma}_{k,m}} & \le \beta_k,~m \in \overline{[M]}. \label{up:tilde}
\end{align}
High confidence set $\cU_{k}$ ensures that the estimate $\btheta_k$ lies in a neighborhood of real transition kernel parameter $\btheta^*$. Here the algorithm adds $2M$ inequalities to constraints in each episode. These inequalities guarantee that estimations of the variance of $\hat{V}_{k,h}$ and $\tilde{V}_{k,h}$ up to $M$-th order are near the real values.

\paragraph{Planning Phase}
After finishing the exploration, the agent enters the planning phase and receives the real reward function. Depending on optimal Bellman equations, the agent is able to obtain the optimal policy backward from state $H$ to state $1$ by dynamic programming based on real reward function $r$ and transition kernel parameter estimate $\btheta_K$. And then, the algorithm outputs the optimal policy.

\paragraph{Computational Complexity of \textbf{HF-UCRL-RFE++}}
Similar with~\citet{chen2022nearoptimal}, we assume that the optimization over $\btheta$, $\pi$, and $r$ in Line~\ref{ln:optimization} of Algorithm~\ref{alg:exp} can be accomplished with an oracle which is obvious to be called for $K$ times. At each episode $k$ and each stage $h$, \textbf{HF-UCRL-RFE++} computes $\{\hat\bphi_{k, h, m}\}_{m\in\overline{[M]}}$, $\{\tilde\bphi_{k, h, m}\}_{m\in\overline{[M]}}$, $\big\{\hat \sigma_{k, h, m}\big\}_{m\in\overline{[M]}}$, $\big\{\tilde \sigma_{k, h, m}\big\}_{m\in\overline{[M]}}$, and updates $\{\hat \bSigma_{k, h, m}\}_{m\in\overline{[M]}}$, $\{\tilde \bSigma_{k, h, m}\}_{m\in\overline{[M]}}$. The computation of $\{\hat\bphi_{k, h, m}\}_{m\in\overline{[M]}}$ and $\{\tilde\bphi_{k, h, m}\}_{m\in\overline{[M]}}$ require $O(\cO M)$ times. According to Algorithm~\ref{alg:home}, calculating $\big\{\hat \sigma_{k, h, m}\big\}_{m\in\overline{[M]}}$ and $\big\{\tilde \sigma_{k, h, m}\big\}_{m\in\overline{[M]}}$ require $O(Md^2)$ time since the computation of the inner-product an inversion of matrix and a vector needs $O(d^2)$. The updates of $\{\hat \bSigma_{k, h, m}\}_{m\in\overline{[M]}}$ and $\{\tilde \bSigma_{k, h, m}\}_{m\in\overline{[M]}}$ further require $O(Md^2)$ time. Lastly, determining the optimal policy during the planning phase takes $O(H(SAd + \cO))$ time. Therefore, the total time complexity of \textbf{HF-UCRL-RFE++} is $O(KH(\cO M + Md^2) + HSAd)$.

\section{Main Results}
\subsection{Upper Bounds}
We first provide the suboptimality upper bound of our algorithm HF-UCRL-RFE++.
\begin{theorem}\label{thm:main}
For Algorithm~\ref{alg:exp}, set $M = \log(7KH)/\log(2)$, $\alpha = H^{-1/2}$, $\gamma = d^{-1/4}$, $\lambda = d/B^2 $, $\{\beta_k\}_{k\ge1}$ as 
\begin{align}
\beta_k =  12 \sqrt{d \eta\tau}
 + 30 \tau /\gamma^2 + \sqrt{\lambda}B \notag,
\end{align}
and denote $\beta=\beta_K$, where $\eta = \log(1 + kH/(\alpha^2 d\lambda))$ and $\tau = \log(32(\log(\gamma^2/\alpha)+1)k^2H^2/\delta)$. Then, for any $0<\delta<1$, we have with probability at least $1-\delta$, after collecting $K$ episodes of samples, algorithm~\ref{alg:exp} returns a policy $\hat\pi_r$ satisfying the following sub-optimality bound,
\begin{align}
V_1^*(s_1;r) - V_1(s_1;\btheta^*, \hat\pi_r, r) = \tilde{O}\left(\frac{d^2}{K} + \frac{d}{\sqrt{K}}\right). \notag
\end{align}
\end{theorem}

The next corollary specifies the sample complexity of our algorithm.
\begin{corollary}\label{cor:samplecomplexity}
Under the same conditions as in Theorem~\ref{thm:main}, Algorithm~\ref{alg:exp} has sample complexity of $m(\varepsilon,\delta) = \tilde{O}(d^2\varepsilon^{-2})$ to output an $\varepsilon$-optimal policy in the planning phase. The exact expression of sample complexity is delayed to Appendix in Lemma~\ref{lm:exactSC}.
\end{corollary}

\begin{remark}
To the best of our knowledge, Corollary~\ref{cor:samplecomplexity} provides the first horizon-free sample complexity upper bound independent of state space size $S$ and action space size $A$ for reward-free exploration. This result shows that long-horizon planning does not add extra difficulty to reward-free exploration.
\end{remark}

\begin{corollary}
\label{cor:samplecomplexityH}
When rescaling the assumption $ \sum_{h=1}^Hr_h(s_h,a_h) \le 1$ to $\sum_{h=1}^Hr_h(s_h,a_h) \le H$, under the same conditions as in Theorem~\ref{thm:main}, Algorithm~\ref{alg:exp} has sample complexity of $m(\varepsilon, \delta) = O(H^2d^2\varepsilon^{-2})$ to output an $\varepsilon$-optimal policy in the planning phase.
\end{corollary}
\begin{remark}
The assumption $\sum_{h=1}^Hr_h(s_h,a_h) \le H$ covers the vanilla assumption $r_h(s_h,a_h)\in[0,1]$. Therefore, compared with \citet{chen2022nearoptimal}, our analysis does not require the $d > H$ assumption and achieves the same sample complexity bound up to logarithmic factors except for the trivial $\tilde O(H)$ difference between time-homogeneous and time-inhomogeneous models with a milder assumption. This improvement can be attributed to the refined value target regression technique, high-order moment estimation (HOME), adopted in our approach. We provide a detailed analysis of this improvement in the ``High-order Moment Estimation" part in the Section~\ref{sec:alg}.

\end{remark}
\subsection{Lower Bounds}
\label{se:lower_bound}
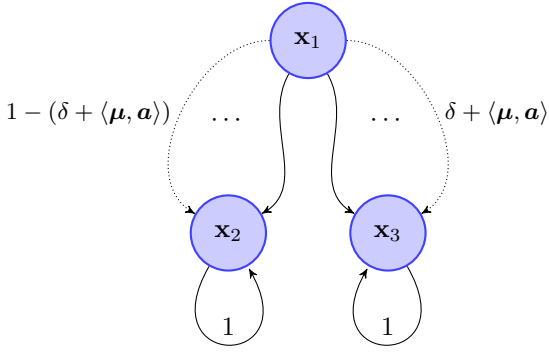
\begin{figure}
\centering
\begin{tikzpicture}[node distance=1.5cm,>=stealth',bend angle=45,auto]
\tikzstyle{place}=[circle,thick,draw=blue!75,fill=blue!20,minimum size=10mm]
\begin{scope}
\node [place, label] (c0){$\xb_1$};
\coordinate [below left of = c0, label = center:{$\cdots$}] (l1) {};
\coordinate [below right of = c0, label = center:{$\cdots$}] (r1) {};
\node [place] (l2) [below of=l1]{$\xb_{2}$};
\node [place] (r2) [below of=r1]{$\xb_{3}$};
\path[->] (c0)
    edge [in=150,out=180, densely dotted] node[left]{\footnotesize{$1-(\delta + \la\bmu,\ba\ra)$}} (l2)
    edge [in=30,out=-120] node[]{} (l2)
    edge [in=30,out=0, densely dotted] node[right, align = center]{\footnotesize{$\delta + \la\bmu,\ba\ra$}} (r2)
    edge [in=150,out=-60] node[]{} (r2)
    ;
\path[->] (l2)
    edge [in=-60,out=-120, min distance = 6cm, loop] node[above] {$1$} ()
    ;
\path[->] (r2)
    edge [in=-120,out=-60, min distance = 6cm, loop] node[above] {$1$} ()
    ;
\end{scope} 
\end{tikzpicture}
\caption{The transition kernel of the hard-to-learn linear mixture MDPs.}
\label{fig:hardmdp}
\end{figure}

The following results provide lower bounds of the sample complexity and suggest that our algorithm is minimax optimal. We will consider the \textit{hard-to-learn linear mixture MDPs} constructed in \citet{zhou22high}. The state space is $\cS = \{x_1, x_2, x_3\}$ and the action space is $\cA=\{ \textbf{a}\} = \{-1, 1\}^{d-1}$. The reward function satisfies $r(x_1,\cdot) = r(x_2,\cdot) = 0$, and $r(x_3,\cdot) = \frac{1}{H}$. The transition probability is defined to be $\PP(x_2 \mid x_1,\ba) = 1-(\delta + \la\bmu,\ba\ra)$ and $\PP(x_3 \mid x_1,\ba) = \delta + \la\bmu,\ba\ra$, where $\delta = 1/6$ and $\bmu \in \{ -\Delta, \Delta \}^{d-1}$ with $\Delta = \sqrt{\delta/K}/(4\sqrt{2})$.

\begin{theorem}\label{thm:lowerbound}
Suppose $B>1$. Then for any algorithm $\texttt{ALG}_{Free}$ solving reward-free linear mixture MDP problems satisfying assumption~\ref{as:reward}, there exist a linear mixture MDP $\cM$ such that $\texttt{ALG}_{Free}$ needs to collect at least $\Omega\left(d^2\varepsilon^{-2}\right)$ episodes of samples to output an $\varepsilon$-optimal policy with probability at least $1-\delta$. This lower bound matches the sample complexity upper bound provided in Corollary~\ref{cor:samplecomplexity}, which shows our upper bound is optimal.
\end{theorem}
\begin{remark}
The lower bound is similar to the lower bound provided in \citet{chen2022nearoptimal}. The first difference is that we rescale the non-zero reward in hard-to-learn cases from $1$ to  $\frac{1}{H}$ in order to satisfy Assumption~\ref{as:reward}. The second difference is that we consider the time-homogeneous model instead of the time-inhomogeneous one in theirs. By these changes, our lower bound for reward-free exploration provided in Theorem~\ref{thm:lowerbound} removes the unnecessary polynomial dependency on episode length $H$ introduced by the scale of total reward.
\end{remark}

\begin{corollary}\label{cor:lowerboundH}
Under the same conditions as Theorem~\ref{thm:lowerbound} except replacing $\sum_{h=1}^H r_h(s_h,a_h) \le 1$ with $r_h\in[0,1]$, for any algorithm $\texttt{ALG}_{Free}$ solving reward-free linear mixture MDP problems satisfying assumption~\ref{as:reward}, there exist a linear mixture MDP $\cM$ such that $\texttt{ALG}_{Free}$ needs to collect at least $\tilde{\Omega}\left(H^2d^2\varepsilon^{-2}\right)$ episodes to output an $\varepsilon$-optimal policy with probability at least $1-\delta$. {This means our upper bound is optimal.}
\end{corollary}

\section{Proof Sketch of Theorem~\ref{thm:main}} \label{sc:main_proof}
We provide the proof sketch of Theorem~\ref{thm:main} along with several key lemmas in big-O notation. The detail of these lemmas is restated in Appendix~\ref{app:main_results}. The following lemmas are conditioned on some good events.

Firstly, Lemma~\ref{lm:suboptimality_main} controls the suboptimality gap between optimal value functions and our estimated value function in the planning phase with the uncertainty along trajectories.

\begin{lemma}\label{lm:suboptimality_main}
For any reward function $r$ in the planning phase, the suboptimality gap of the outputted policy $\hat\pi_r$ can be bounded as
\begin{align}
V_1^*\left(s_1;r\right) - V_1\left(s_1;\btheta^*, \hat\pi_r, r\right) \le 4 \hat V_{K,1}\left(s_1\right). \label{eq:mian_suboptimality}
\end{align}
\end{lemma}

Then, the next lemma shows that the uncertainty along trajectories decreases with respect to episodes. This lemma is intuitively right since the uncertainty should decrease with more information collected.

\begin{lemma}\label{lm:decreasing_main}
For uncertainty along trajectories, we have 
\begin{align}
\hat V_{K,1}(s_1;\btheta_K,\pi_K,r_K) \le \frac{1}{K}\bigg(\sum_{k=1}^K \hat V_{k,1}(s_1;\btheta_k,\pi_k,r_k )\bigg). \notag
\end{align}
\end{lemma}

The last lemma upper bounds the sum of the uncertainty along trajectories.
\begin{lemma}\label{lm:sum_main}
For any $0<\delta<1$, with probability at least $1-4M\delta$, we have
\begin{align}
\sum_{k=1}^K \hat V_{k,1}(s_1;\btheta_k,\hat\pi_k,r_k ) = \tilde{O}(d\sqrt{K} + d^2). \label{eq:main_sum}
\end{align}
\end{lemma}

Equipped with the above lemmas, we are ready to prove Theorem~\ref{thm:main}.

\begin{proof}[Proof of Theorem~\ref{thm:main}]
The suboptimality of the policy output in the planning phase can be bounded by the uncertainty along trajectories in the last episode of the exploration phase as the following equation according to Lemma~\ref{lm:suboptimality_main}.
\begin{align}
& V_1^*(s_1;r) - V_1(s_1;\btheta^*, \hat\pi_r, r)   \le 4 \hat V_{K,1}(s_1).   
\end{align}
Since Lemma~\ref{lm:decreasing_main} indicates that the uncertainty is decreasing with the episodes, the uncertainty of the last episode can be further upper bounded by the sum of uncertainty in each episode by substituting \eqref{eq:mian_suboptimality} in Lemma~\ref{lm:suboptimality_main} into the above inequality: 
\begin{align}
 V_1^*(s_1;r) - V_1(s_1;\btheta^*, \hat\pi_r, r)    
 \le \frac{4}{K} \sum_{k=1}^K \hat V_{k,1}(s_1). \label{eq:main-proof}
\end{align}
At last, the sum of value functions can be upper bounded according to Lemma~\ref{lm:sum_main}. Thus, plugging \eqref{eq:main_sum} into the~\eqref{eq:main-proof} as follows concludes our proof.
\begin{align}
V_1^*(s_1;r) - V_1(s_1;\btheta^*, \hat\pi_r, r) = \tilde{O}\left(\frac{d^2}{K} + \frac{d}{\sqrt{K}}\right). \notag
\end{align}
\end{proof}

\section{Conclusion}
We study model-based reward-free exploration for learning the linear mixture MDPs. Our algorithm is guaranteed to have horizon-free sample complexity in the exploration phase to find a near-optimal policy in the planning phase for any given reward function. To our knowledge, these are the first horizon-free result for reward-free RL with function approximation. We also give a sample complexity lower bound for reward-free exploration in linear mixture MDPs under our assumptions. The sample complexity upper bound of our algorithm matches the lower bound up to logarithmic factors.


\section*{Acknowledgements}

We thank the anonymous reviewers for their helpful comments. WZ, JZ and QG are supported in part by the National Science Foundation CAREER Award 1906169 and research fund from UCLA-Amazon Science Hub. The views and conclusions contained in this paper are those of the authors and should not be interpreted as representing any funding agencies.

\bibliography{ref}
\bibliographystyle{icml2023}


\appendix
\onecolumn


\section{Omitted Proof in Section~\ref{sc:main_proof}} \label{app:main_results}
In this section, we provide detailed proof for Theorem~\ref{thm:main}. For $k\in[K]$, $h\in[H]$, let $\cF_{k,h}$ be the $\sigma$-algebra generated by the random variables representing the state-action pairs up to and including those that appear stage $h$ of episode $k$. That is, $\cF_{k,h}$ is generated by
    
\begin{center}
\begin{tabular}{ccccc}
        $s_1^1,a_1^1$ & $\cdots$ & $s_h^1,a_h^1$ & $\cdots$ & $s_H^1,a_H^1$ \\
        $s_1^2,a_1^2$ & $\cdots$ & $s_h^2,a_h^2$ & $\cdots$ & $s_H^2,a_H^2$ \\ 
        $\vdots$&&$\vdots$&& \\
        $s_1^k,a_1^k$ & $\cdots$ & $s_h^k,a_h^k.$ &  &
\end{tabular}
\end{center}

\subsection{Notations}
To establish clarity and facilitate understanding, we provide the Table~\ref{tab:notations} that outlines the notations which will be utilized throughout the proof.

\begin{table}[ht]
\centering
\begin{tabular}{cl}
\hline
Notation & Meaning \\
\hline
$s_h,a_h$ & States and actions introduced by a general policy $\pi$ (not specified).\\
\rowcolor{LightGray}
$s_h^k,a_h^k$ & States and actions introduced in the episode $k$ by the policy $\pi_k$. \\
$u_{k,h}(s,a;\btheta,\pi,r)$ & The uncertainty of states and actions, defined in Equation~\eqref{eq:pseudo_r}. \\
\rowcolor{LightGray}
\Gape[0pt][2pt]{\makecell[c]{$\btheta_k$,\quad$\pi_k$,\\$r_{k}=\{r_{k,h}\}_{h\in[H]}$}} & \Gape[0pt][2pt]{\makecell[l]{The transition kernel parameter, the exploration policy, and the pseudo\\  reward function obtained  via the optimization oracle in\\ Line~\ref{ln:optimization} of Algorithm~\ref{alg:exp}.}}\\
$V_h(s; \btheta, \pi, r)$ & \makecell[l]{The value function of policy $\pi$  in the linear mixture MDP \\ with transition kernel parameter $\btheta$ and reward function $r$.} \\
\rowcolor{LightGray}
$\hat{V}_{k,h}(s;\btheta,\pi,r)$ & \Gape[0pt][2pt]{\makecell[l]{The uncertainty along the trajectory, defined in Equation~\eqref{eq:pseudo_V}\\ }} \\
$\tilde{V}_{k,h}(s)$ & The pseudo value function, equal to $V_h(s; \btheta_k, \pi_k, r_k)$ \\
\rowcolor{LightGray}
$\hat\btheta_{k,m}$ & \Gape[0pt][2pt]{\makecell[l]{The estimated transition kernel parameter obtained \\by value regression targeting $\hat{V}_{k,h}^{2^m}$.}} \\
$\tilde\btheta_{k,m}$ & \Gape[0pt][2pt]{\makecell[l]{The estimated transition kernel parameter obtained \\by value regression targeting $\tilde{V}_{k,h}^{2^m}$.}} \\
\rowcolor{LightGray}
$\btheta^*$ & The ground-truth transition kernel parameter. \\
$\cU_k$ & The confidence set containing $\btheta^*$ with high probability. \\
\rowcolor{LightGray}
$\beta_k$, ($\beta=\beta_K$) & The radius of confidence set $\cU_k$. \\
$\tilde\Sigma_{k,h,m}$, $\hat\Sigma_{k,h,m}$ & The covariance matrix for $\tilde{V}^{2^m}_{k,h}$ and $\hat{V}^{2^m}_{k,h}$, respectively.\\
\rowcolor{LightGray}
$\dot{\tilde\Sigma}_{k,m}$, $\dot{\hat\Sigma}_{k,m}$ & Equal to $\tilde\Sigma_{k-1,H+1,m}$ and $\hat\Sigma_{k-1,H+1,m}$, respectively.\\
$\tilde\sigma_{k,h,m}$, $\hat\sigma_{k,h,m}$ & \Gape[0pt][2pt]{\makecell[l]{The weights for regression problems targeting $\tilde{V}_{k,h}^{2^m}$ and $\hat{V}_{k,h}^{2^m}$\\  respectively, defined in Equation \eqref{eq:weight}.}} \\
\rowcolor{LightGray}
$\tilde \bphi_{k, h, m}$, $\hat \bphi_{k, h, m}$ & Equal to $\bphi_{\hat V_{k,h+1}^{2^m}}(s_h^k, a_h^k)$ and $ \bphi_{\tilde V_{k,h+1}^{2^m}}(s_h^k, a_h^k)$, respectively.\\
$\hat\pi_r$ & The policy obtained in the planning phase.
\\
\hline
\end{tabular}   
\caption{Important Notations}
\label{tab:notations}
\end{table}

\subsection{Proof of Lemma~\ref{lm:suboptimality_main}}

\begin{lemma}\label{lm:main_theta}
For all $0 < \delta < 1$, suppose $\beta_k$ is set as in Theorem~\ref{thm:main}, the following event happens with probability at least $1 - 2M\delta$
\begin{align}
        \left \| \hat \btheta_{k,m} - \btheta^* \right\|_{\dot{\hat\bSigma}_{k,m}} &\le \beta_k \label{eq:hattheta} \\
        \left \| \tilde \btheta_{k,m} - \btheta^* \right\|_{\dot{\tilde\bSigma}_{k,m}} &\le \beta_k \label{eq:tildetheta} \\
        \left \| \btheta_k - \btheta^* \right\|_{ \dot{\hat\bSigma}_{k,0}} &\le 2 \beta_k   \label{eq:theta_hat} \\
        \left\| \btheta_k - \btheta^* \right \|_{ \dot{\tilde\bSigma}_{k,0}} &\le 2 \beta_k . \label{eq:theta_tilde}
    \end{align}
\end{lemma}

In the following,  we define the event that Lemma~\ref{lm:main_theta} holds to be $\cE_{\ref{lm:main_theta}}$. And the following lemmas are conditioned on $\cE_{\ref{lm:main_theta}}$ by default. We define function $W_h$ for certain sequence $\{ R_h \}$ recursively as
    \begin{align}
        W_h\left(\left\{R_h\right\}\right) = \min\left\{1, R_h + W_{h+1}\left(\{R_h\}\right)\right\}. \notag
    \end{align}
In addition we denote the trajectory of first $h$ steps as $ \texttt{traj}_h:=(s_1,a_1,\cdots,s_{h-1},a_{h-1},s_h)$, and the trajectory sampled from $(\pi,\PP)$ conditioned on $\texttt{traj}_h$ as $\texttt{traj}\sim (\pi,\PP)|\texttt{traj}_h$.
\begin{lemma}\label{lm:traj}
For any policy $\pi$ and reward function $r \in R$,
we have 
\begin{align}
V_1(s_1;\btheta_K,\pi,r) - V_1(s_1;\btheta^*,\pi,r)
        = \EE_{\texttt{traj}\sim(\pi,\PP)|\texttt{traj}_1} W_1( \{ (\PP_K - \PP) V_{h+1}(s_h;\btheta_K,\pi,r) \}) \label{eq:traj}
\end{align}
\end{lemma}

\begin{lemma}\label{lm:optimism}
For any policy $\pi$ and reward function $r \in R$, we have
\begin{align}
\EE_{\texttt{traj}\sim(\pi,\PP)|\texttt{traj}_1} W_1\left(\left\{ u_{k,h}(s_h,\pi(s_h);\btheta_K, \pi, r) \right\}\right) \le \hat V_{K,1}\left(s_1;\btheta_K,\pi_K,r_K\right). \notag
\end{align}
\end{lemma}

\begin{proof}[Proof of Lemma~\ref{lm:suboptimality_main}]
The proof follows the proof of Lemma 15 in \citet{zhang2020nearly}. Firstly,
    \begin{align}
        &V_1^*(s_1;r) - V_1(s_1;\btheta^*, \hat\pi_r, r)  \notag \\
        & = (V_1^*(s_1;r) - V_1(s_1;\btheta_K,\hat \pi_r ,r)) + (V_1(s_1;\btheta_K,\hat \pi_r,r ) -  V_1(s_1;\btheta^*, \hat\pi_r, r))  \notag \\
        & \le (V_1^*(s_1;r) - V_1(s_1;\btheta_K,\pi_r^*,r )) + (V_1(s_1;\btheta_K,\hat \pi_r ,r) -  V_1(s_1;\btheta^*, \hat\pi_r, r)), \label{eq:proofA1eq1}
    \end{align}
where $\pi_r^*$ is the optimal policy for $(\btheta, r)$, and $ \hat\pi_r$ is the optimal policy for $(\btheta_K, r)$. Then for any policy $\pi\in\Pi$,
    \begin{align}
        &\left|V_1(s_1;\btheta_K,\pi,r) - V_1(s_1;\btheta^*,\pi,r)\right| \notag \\
        & = \left| \EE_{\texttt{traj}\sim(\pi,\PP)|\texttt{traj}_1} W_1\left( \left\{ (\PP_K - \PP) V_{h+1}(s_h,a_h;\btheta_K,\pi,r) \right\}\right)\right| \notag \\
        & =  \left| \EE_{\texttt{traj}\sim(\pi,\PP)|\texttt{traj}_1} W_1\left( \left\{ (\btheta_K - \btheta^*) \bphi_{V_{h+1}(\cdot;\btheta_K,\pi,r)}(s_h,a_h) \right\}\right)\right| \notag \\
        &\le \EE_{\texttt{traj}\sim(\pi,\PP)|\texttt{traj}_1} W_1\left(\left\{ \left\|\btheta_K - \btheta^*\right\|_{\dot{\tilde\bSigma}_{k,0}}  \left\| \bphi_{V_{h+1}(\cdot;\btheta_K,\pi,r)}(s_h,a_h) \right\|_{\dot{\tilde\bSigma}_{k,0}^{-1}}\right\}\right)  \notag \\
        & \le \EE_{\texttt{traj}\sim(\pi,\PP)|\texttt{traj}_1} W_1\left(\left\{ 2\beta  \left\| \bphi_{V_{h+1}(\cdot;\btheta_K,\pi,r)}(s_h,a_h) \right\|_{\dot{\tilde\bSigma}_{k,0}^{-1}}\right\}\right)  \notag \\
        & = 2  \EE_{\texttt{traj}\sim(\pi,\PP)|\texttt{traj}_1} W_1\left(\left\{ u_h(s_h,a_h;\btheta_K, \pi, r) \right\}\right)  \notag \\
        & \le 2 \hat V_{K,1}(s_1;\btheta_K,\pi_K,r_K). \label{eq:proofA1eq2}
    \end{align}
    The first equality holds due to Lemma~\ref{lm:traj}, the second inequality holds due to Cauchy-Schwartz inequality, the third inequality holds due to Lemma~\ref{lm:main_theta}, and the last inequality holds due to Lemma~\ref{lm:optimism}.
Plugging \eqref{eq:proofA1eq2} into \eqref{eq:proofA1eq2}, we obtain
\begin{align}
  V_1^*(s_1;r) - V_1(s_1;\btheta^*, \hat\pi_r, r) &\le 2 \hat V_{K,1}(s_1;\btheta_K,\pi_K,r_K) + 2 \hat V_{K,1}(s_1;\btheta_K,\pi_K,r_K) \notag \\
 & = 4 \hat V_{K,1}(s_1;\btheta_K,\pi_K,r_K). \notag
\end{align}
\end{proof}
\subsection{Proof of Lemma~\ref{lm:decreasing_main}}
\begin{proof}[Proof of Lemma~\ref{lm:decreasing_main}]
    The proof follows the proof of Lemma 14 in \citet{chen2022nearoptimal}. Firstly, we prove that $ \hat V_{k,1}(s;\btheta,\pi,r ) $ is non-increasing w.r.t. $k$ for any fixed $\btheta,\pi,r $ by induction in $h$. Suppose for any $k_1\le k_2$, $ \hat V_{k_1,h+1}(s;\btheta,\pi,r ) \ge \hat V_{k_2,h+1}(s;\btheta,\pi,r )$ for any $s$. By definition,
    \begin{align}
        \hat V_{k,h}(s;\btheta,\pi,r) & = \min \left\{1, u_{k,h}(s,a;\btheta,\pi, r) + 2 \beta  \left\| \bphi_{\hat V_{k,h+1}(\cdot;\btheta,\pi,r)} (s,\pi(s))\right\|_{\dot{\hat\bSigma}_{k,0}^{-1}}\right.\notag\\
        &\qquad\qquad\qquad\qquad\qquad\qquad\qquad + \bphi^\top_{\hat V_{k,h+1}(\cdot;\btheta,\pi,r)} (s,\pi(s))\btheta\bigg\}  \notag \\
        u_{k,h}(s,a;\btheta,\pi, r) & =  \beta\left\| \bphi_{V_h(\cdot;\btheta,\pi,r)}(s,a) \right\|_{\dot{\tilde\bSigma}_{k,0}^{-1}} \notag
    \end{align}
    Since $\dot{\hat\bSigma}_{k_1,0} \preceq \dot{\hat\bSigma}_{k_2,0}$ and $\dot{\tilde \bSigma}_{k_1,0} \preceq \dot{\tilde\bSigma}_{k_2,0}$, we have
    \begin{align}
        u_{k_1,h}(s,a;\btheta,\pi, r) &\ge u_{k_2,h}(s,a;\btheta,\pi, r)  \notag \\
        \left\| \bphi_{\hat V_{k_1,h+1}(\cdot;\btheta,\pi,r)} (s,\pi(s))\right\|_{\dot{\hat\bSigma}_{k,0}^{-1}} &\ge \left\| \bphi_{\hat V_{k_2,h+1}(\cdot;\btheta,\pi,r)} (s,\pi(s))\right\|_{\dot{\hat\bSigma}_{k,0}^{-1}} \notag \\
        \bphi^\top_{\hat V_{k_1,h+1}(\cdot;\btheta,\pi,r)} (s,\pi(s))\btheta &\ge \bphi^\top_{\hat V_{k_2,h+1}(\cdot;\btheta,\pi,r)} (s,\pi(s))\btheta \notag
    \end{align}
    Thus $ \hat V_{k_1,h}(s;\btheta,\pi,r ) \ge \hat V_{k_2,h}(s;\btheta,\pi,r )$ for any $k_1\le k_2$. Furthermore, since $\cU_{k_2} \subset \cU_{k_1}$, and $\btheta_k,\pi_k,r_k$ are argmax over $\cU_{k}$, we have
    \begin{align}
        \hat V_{k_1,1}(s_1;\btheta_{k_1},\pi_{k_1},r_{k_1}) \ge \hat V_{k_1,1}(s_1;\btheta_{k_2},\pi_{k_2},r_{k_2}) \ge \hat V_{k_2,1}(s_1;\btheta_{k_2},\pi_{k_2},r_{k_2}) \notag
    \end{align}
    It follows that $\hat V_{k,1}(s_1^k;\btheta_k,\pi_k,r_k )$ is non-increasing w.r.t. $k$. Thus,
    \begin{align}
        K \hat V_{K,1}(s_1;\btheta_K,\pi_K,r_K) \le \sum_{k=1}^K \hat V_{k,1}(s_1;\btheta_k,\pi_k,r_k ) \notag
    \end{align}
\end{proof}

\subsection{Proof of Lemma~\ref{lm:sum_main}}
\begin{lemma}\label{lm:ucb}
Conditioned on the event $\cE$, let $\tilde V_{k,h}$, $\hat V_{k,h}$, $\dot{\tilde\bSigma}_{k,m}$, $\dot{\hat\bSigma}_{k,m}$, $\tilde \bphi_{k,h,m}$, $\hat \bphi_{k,h,m}$ be defined in Algorithm~\ref{alg:exp}, for any $k\in[K]$, $h\in[H]$, $m\in\overline{[M]}$, we have
\begin{align}
    \hat V_{k,h}(s_h^k) - u_{k,h}(s_h^k, a_h^k) - \PP \hat V_{k,h+1}(s_h^k,a_h^k) \le 4 \min \left\{1, \beta \left\| \hat \bphi_{k,h,0} \right\|_{\dot{\hat\bSigma}_{k,0}^{-1}} \right\} \label{eq:hatucb} \\
    \tilde V_{k,h}(s_h^k) - r_{k,h}(s_h^k, a_h^k) - \PP \tilde V_{k,h+1} \le 2\min\left\{1, \beta\left\|\tilde \bphi_{k,h,0} \right\|_{\dot{\tilde\bSigma}_{k,0}^{-1}} \right\} \label{eq:tildeucb}
\end{align}
\end{lemma}

In order to prove Lemma~\ref{lm:sum_main}, we introduce the following quantities used in~\citet{zhou22high} as
\begin{align}
    \hat R_m &= \sum_{k=1}^K\sum_{h=1}^H I_h^j\min\left\{1, \beta\|\hat\bphi_{k, h, m}\|_{\dot{\hat\bSigma}_{k, m}^{-1}}\right\}, \forall m \in \overline{[M]} \label{eq:hatR}\\
    \tilde R_m &= \sum_{k=1}^K\sum_{h=1}^H I_h^j\min\left\{1, \beta\|\tilde \bphi_{k, h, m}\|_{\dot{\tilde\bSigma}_{k, m}^{-1}}\right\}, \forall m \in \overline{[M]} \label{eq:tildeR}\\
    \hat A_m &= \sum_{k=1}^K\sum_{h=1}^H I_h^k\left[\left[\PP \hat V_{k,h+1}^{2^m}\right]\left(s_h^k, a_h^k\right) - \hat V_{k,h+1}^{2^m}\left(s_{h+1}^k\right)\right], \forall m \in \overline{[M]} \label{eq:hatA}\\
    \tilde A_m &= \sum_{k=1}^K\sum_{h=1}^H I_h^k\left[\left[\PP \tilde V_{k,h+1}^{2^m}\right]\left(s_h^k, a_h^k\right) - \tilde V_{k,h+1}^{2^m}\left(s_{h+1}^k\right)\right], \forall m \in \overline{[M]} \label{eq:tildeA}\\    
    \hat S_m &= \sum_{k=1}^K\sum_{h=1}^HI_h^k\left[\VV \hat V_{k,h+1}^{2^m}\right]\left(s_h^k, a_h^k\right), \forall m \in \overline{[M]} \label{eq:hatS}\\
    \tilde S_m &= \sum_{k=1}^K\sum_{h=1}^HI_h^k\left[\VV \tilde V_{k,h+1}^{2^m}\right]\left(s_h^k, a_h^k\right), \forall m \in \overline{[M]} \label{eq:tildeS}\\
    I_h^k &= \ind\left\{\forall m \in \overline{[M]}, \det \left(\dot{\hat\bSigma}_{k,  m}^{-1/2}\right) / \det \left(\hat \bSigma_{k, h, m}^{-1/2}\right) \le 4~\text{and}~\det \left(\dot{\tilde \bSigma}_{k, m}^{-1/2}\right) / \det \left(\tilde \bSigma_{k, h, m}^{-1/2}\right) \le 4\right\} \label{eq:I}\\
    G &= \sum_{k=1}^K\left(1 - I_H^k\right), \label{eq:G}
\end{align}

\begin{lemma}\label{lm:R}
Let $\gamma$, $\alpha$, be defined in Algorithm~\ref{alg:home}, $\{\hat 
R_m\}_{m\in\overline{[M]}}$, $\{\tilde 
R_m\}_{m\in\overline{[M]}}$, $\{\hat S_m\}_{m\in\overline{[M]}}$, $\{\tilde S_m\}_{m\in\overline{[M]}}$ be defined in \eqref{eq:hatR}, \eqref{eq:tildeR}, \eqref{eq:hatS}, \eqref{eq:tildeS}. Then for $m \in \overline{[M-1]}$, we have
\begin{align}
    \hat R_m \le \min\left\{ KH, 4d\iota + 4\beta \gamma^2 d\iota + 2 \beta \sqrt{d\iota}\sqrt{\hat S_m+4\hat R_m+2\hat R_{m+1} +KH\alpha^2}\right\}  \label{eq:inhatR} \\
    \tilde R_m \le \min\left\{ KH, 4d\iota + 4\beta \gamma^2 d\iota + 2 \beta \sqrt{d\iota}\sqrt{\tilde S_m+4\tilde R_m+2\tilde R_{m+1} +KH\alpha^2}\right\}, \label{eq:intildeR}
\end{align}
where $\iota = \log(1 + KH/(d\lambda\alpha^2))$. For $\hat R_{M-1}$ and $\tilde R_{M-1}$, we have the trivial bound $\hat R_{M-1} \le KH$ and $\tilde R_{M-1} \le KH$.
\end{lemma}

\begin{lemma}\label{lm:S}
Let  $\{\hat 
R_m\}_{m\in\overline{[M]}}$, $\{\tilde 
R_m\}_{m\in\overline{[M]}}$, $\{\hat S_m\}_{m\in\overline{[M]}}$, $\{\tilde S_m\}_{m\in\overline{[M]}}$, $\{\hat A_m\}_{m\in\overline{[M]}}$, $\{\tilde A_m\}_{m\in\overline{[M]}}$, $G$ be defined as \eqref{eq:hatR}, \eqref{eq:tildeR}, \eqref{eq:hatS}, \eqref{eq:tildeS}, \eqref{eq:hatA}, \eqref{eq:tildeA}, \eqref{eq:G}. Then, conditioned on the event $\cE$, for $m\in\overline{[M-1]}$, we have
\begin{align}
\hat S_m \le \left|\hat A_{m+1}\right| +  G + 2^{m+1}\left(\tilde R_0 + 4 \hat R_0\right) \label{eq:inhatS} \\
\tilde S_m \le \left|\tilde A_{m+1}\right| + G +  2^{m+1}\left(K + 2 \tilde R_0\right)\label{eq:intildeS}
\end{align}
\end{lemma}

\begin{lemma}\label{lm:A}
Let  $\{\hat S_m\}_{m\in\overline{[M]}}$, $\{\tilde S_m\}_{m\in\overline{[M]}}$, $\{\hat A_m\}_{m\in\overline{[M]}}$, $\{\tilde A_m\}_{m\in\overline{[M]}}$ be defined as  \eqref{eq:hatS}, \eqref{eq:tildeS}, \eqref{eq:hatA}, \eqref{eq:tildeA}. Then we have $\PP(\cE_{\ref{lm:A}})>1-2M\delta$, with $\cE_{\ref{lm:A}}$ be defined as,
    \begin{align}
        \cE_{\ref{lm:A}} := \left\{\forall m \in \overline{[M]}, \left|\hat A_m\right| \le \min\left\{\sqrt{2\zeta\hat S_m} + \zeta,   KH\right\}~\text{and}~\left|\tilde A_m\right| \le \min\left\{\sqrt{2\zeta\tilde S_m} + \zeta,   KH\right\}  \right\},\label{eq:inA}
    \end{align}
where $\zeta = 4\log(4\log(KH)/\delta)$.
\end{lemma}

\begin{lemma}\label{lm:G}
Let $G$ be defined in \eqref{eq:G}. Then we have
\begin{align}
G \le Md\iota, \label{eq:inG}
\end{align}
where $\iota = \log\left(1 + KH/\left(d\lambda\alpha^2\right)\right)$.
\end{lemma}

\begin{lemma}\label{lm:sum}\textit{(Restatement of Lemma~\ref{lm:sum_main})}
For any $0<\delta<1$, with probability at least $1-4M\delta$, we have
\begin{align}
&\sum_{k=1}^K \hat V_{k,1}(s_1^k;\btheta_k,\pi_k,r_k )  \notag \\
& \le 896\max\left\{64\beta^2 d \iota, 2\zeta \right\} + 24\zeta + 240d\iota + 240\beta\gamma^2 d \iota  + 120\beta d\iota\sqrt{M} + 24 \sqrt{\zeta Md\iota} + Md\iota\notag \\
&\qquad + \left(64\max\left\{8\beta\sqrt{d\iota}, \sqrt{2\zeta}\right\} + 120\beta \sqrt{d\iota H \alpha^2 }\right)\sqrt{K} \notag
\end{align}
where $\iota = \log(1 + KH/(d\lambda\alpha^2))$, $\zeta = 4 \log(4\log(KH)/\delta)$.
\end{lemma}
\begin{proof}[Proof of Lemma~\ref{lm:sum}]
All the following proofs are conditioned on $\cE_{\ref{lm:main_theta}}\cap\cE_{\ref{lm:A}}$, which happens with probability at least $1-4M\delta$. Firstly, we have
\begin{align}
    & \sum_{k=1}^K \hat V_{k,1}(s_h^k)  \notag \\
    & = \sum_{k=1}^K\sum_{h=1}^H \left[I_h^k \left[\hat V_{k,h}(s_h^k) - \hat V_{k,h+1}(s_{h+1}^k)\right] + \left(1 -  I_h^k\right)\left[\hat V_{k,h}(s_h^k) - \hat V_{k,h+1}(s_{h+1}^k)\right]\right] \notag \\
    & = \sum_{k=1}^K\left[ \sum_{h=1}^H  I_h^k  u_{k,h}(s_h^k,a_h^k) + \sum_{h=1}^H  I_h^k \left[\hat V_{h,k}(s_h^k) - u_{k,h}(s_h^k,a_h^k) - \PP \hat V_{k,h+1}(s_h^k,a_h^k) \right]  \right. \notag \\
    & \quad + \left. \sum_{h=1}^H  I_h^k \left[ \PP \hat V_{k,h+1}(s_h^k,a_h^k) - \hat V_{k,h+1}(s_{h+1}^k)\right] \right] + \sum_{k=1}^K\sum_{h=1}^H (1 -  I_h^k)\left[\hat V_{k,h}(s_h^k) - \hat V_{k,h+1}(s_{h+1}^k)\right] \notag \\
    & \le \underbrace{\sum_{k=1}^K\sum_{h=1}^H  I_h^k u_{k,h}(s_h^k,a_h^k)}_{I_1} + \underbrace{\sum_{k=1}^K\sum_{h=1}^H  I_h^k \left[\hat V_{h,k}(s_h^k) - u_{k,h}(s_h^k,a_h^k) - \PP \hat V_{k,h+1}(s_h^k,a_h^k) \right]}_{I_2}  \notag \\
    & \quad + \underbrace{\sum_{k=1}^K\sum_{h=1}^H  I_h^k \left[ \PP \hat V_{k,h+1}(s_h^k,a_h^k) - \hat V_{h+1,k}(s_{h+1}^k)\right]}_{I_3}  + \underbrace{\sum_{k=1}^K\left(1 -  I_{h_k}^k\right)\hat V_{k,h_k}(s_{h_k}^k)}_{I_4}, \notag
\end{align}
where $h_k$ is the smallest index such that $I^k_{h_k} = 0$. Following the definition of $u_{k,h}$,
\begin{align}
    I_1 = \sum_{k=1}^K\sum_{h=1}^H  I_h^k \min\left\{1, \beta\left\| \tilde \bphi_{k,h,0}  \right\|_{\dot{\tilde\bSigma}_{k,0}^{-1}} \right\} = \tilde R_0. \notag
\end{align}
By Lemma~\ref{lm:ucb},
\begin{align}
    I_2 \le 4 \sum_{k=1}^K\sum_{h=1}^H I_h^k \min\left\{1,\beta\left\| 
 \hat\bphi_{k,h,0}\right\|_{\dot{\hat\bSigma}_{k,0}^{-1}}  \right\} = 4 \hat R_0 \notag
\end{align}
By definitions,
\begin{align}
    I_3 &= \hat A_0,  \notag \\
    I_4 &\le \sum_{k=1}^K \left(1 - I_H^k\right) = G. \notag
\end{align}
Thus,
\begin{align}
    \sum_{k=1}^K \hat V_{k,1}(s_h^k) \le \tilde R_0 + 4 \hat R_0 + \hat A_0 + G \label{eq:sum_medium}
\end{align}
Substituting \eqref{eq:inhatS} in Lemma~\ref{lm:S} into \eqref{eq:inhatR} in Lemma~\ref{lm:R}, we have
\begin{align}
    \hat R_m
    & \le 4d\iota + 4\beta \gamma^2 d\iota + 2 \beta \sqrt{d\iota}\sqrt{\left|\hat A_{m+1}\right| + G + 2^{m+1}\left(\tilde R_0 + 4 \hat R_0\right) + 4\hat R_m+2\hat R_{m+1} +KH\alpha^2}  \notag \\
    & \le 2 \beta \sqrt{d\iota}\sqrt{\left|\hat A_{m+1}\right| +  2^{m+1}\left(\tilde R_0 + 4 \hat R_0\right) + 4\hat R_m+2\hat R_{m+1} } + \underbrace{4d\iota + 4\beta \gamma^2 d\iota + 2 \beta \sqrt{d\iota} \sqrt{G + KH\alpha^2}}_{I_c}, \label{eq:sum_hatR}
\end{align}
where the second inequality holds due to $\sqrt{a+b}\le\sqrt{a}+\sqrt{b}$. Substituting \eqref{eq:inhatS} in Lemma~\ref{lm:S} into \eqref{eq:inA} in Lemma~\ref{lm:A}, we have
\begin{align}
\left|\hat A_m\right| 
& \le \sqrt{2\zeta}\sqrt{\left|\hat A_{m+1}\right| + G + 2^{m+1}\left(\tilde R_0 + 4 \hat R_0\right)} + \zeta  \notag \\
& \le \sqrt{2\zeta}\sqrt{\left|\hat A_{m+1}\right| + 2^{m+1}\left(\tilde R_0 + 4 \hat R_0\right)} + \sqrt{2\zeta G} + \zeta \label{eq:sum_hatA}
\end{align}
Substituting \eqref{eq:intildeS} in Lemma~\ref{lm:S} into \eqref{eq:intildeR} in Lemma~\ref{lm:R}, we have
\begin{align}
    \tilde R_m 
    & \le 4d\iota + 4\beta \gamma^2 d\iota + 2 \beta \sqrt{d\iota}\sqrt{\left|\tilde A_{m+1}\right| + G +  2^{m+1}\left(K + 2 \tilde R_0\right) + 4\tilde R_m+2\tilde R_{m+1} +KH\alpha^2}  \notag \\
    & \le 2 \beta \sqrt{d\iota}\sqrt{\left|\tilde A_{m+1}\right| +   2^{m+1}\left(K + 2 \tilde R_0\right) + 4\tilde R_m+2\tilde R_{m+1} } + \underbrace{4d\iota + 4\beta \gamma^2 d\iota + 2 \beta \sqrt{d\iota} \sqrt{G + KH\alpha^2}}_{I_c} \label{eq:sum_tildeR}
\end{align}
Substituting \eqref{eq:intildeS} in Lemma~\ref{lm:S} into \eqref{eq:inA} in Lemma~\ref{lm:A}, we have
\begin{align}
    \left|\tilde A_m\right| 
    & \le \sqrt{2\zeta}\sqrt{\left|\tilde A_{m+1}\right| + G +  2^{m+1}\left(K + 2 \tilde R_0\right)} + \zeta   \notag \\
    & \le \sqrt{2\zeta} \sqrt{\left|\tilde A_{m+1}\right| + 2^{m+1}\left(K + 2 \tilde R_0\right)} + \sqrt{2\zeta G} + \zeta \label{eq:sum_tildeA}
\end{align}
Thus, calculating \eqref{eq:sum_tildeR} + \eqref{eq:sum_tildeA} + 4$\times$\eqref{eq:sum_hatR} + \eqref{eq:sum_hatA} and using $\sqrt{a} + \sqrt{b} + \sqrt{c} + \sqrt{d} \le 2\sqrt{a + b + c + d}$, we have
\begin{align}
    & \tilde R_m + \left|\tilde A_m\right| + 4  \hat R_m + \left|\hat A_m\right|   \notag \\
    & \le 5 I_c + 2\sqrt{2\zeta G} + 2\zeta + 2\max\left\{8 \beta \sqrt{d\iota} ,\sqrt{2\zeta}\right\}\sqrt{2\left|\hat A_{m+1}\right| +  2\cdot2^{m+1}\left(\tilde R_0 + 4 \hat R_0\right)}   \notag \\ 
    & \qquad \overline{ + 4\hat R_m + 2\hat R_{m+1} + 2 \left|\tilde A_{m+1}\right| +   2\cdot2^{m+1}\left(K + 2 \tilde R_0\right) + 4\tilde R_m + 2\tilde R_{m+1}}   \notag \\
    & \le 5 I_c + + 2\sqrt{2\zeta G} + 2\zeta + 4\max\left\{8 \beta \sqrt{d\iota} ,\sqrt{2\zeta}\right\} \sqrt{\left(\tilde R_m + \left|\tilde A_m \right| + 4 \hat R_m + \left| \hat A_m \right|\right) }   \notag \\
    & \qquad   \overline{+ \left(\tilde R_{m+1} + \left|\tilde A_{m+1}\right| + 4 \hat R_{m+1} + \left|\hat A_{m+1}\right|\right) + 2\cdot2^{m+1}\left(K + \tilde R_0 + \left|\tilde A_0\right| + 4 \hat R_0 + \left|\hat A_0\right|\right)}.  \notag 
\end{align}
Then by Lemma~\ref{lm:sequence} with $a_m = \tilde R_m + \left|\tilde A_m\right| + 4 \hat R_m + \left|\hat A_m\right| \le 7KH$ and $M = \log(7KH)/\log 2$, $\tilde R_0 + \left|\tilde A_0\right| + 4 \hat R_0 + \left|\hat A_0\right|$ can be bounded as
\begin{align}
    &\tilde R_0 + \left|\tilde A_0\right| + 4 \hat R_0 + \left|\hat A_0\right|    \notag \\
    & \le 22 \cdot 16\max\{64\beta^2 d\iota, 2\zeta  \} + 30 I_c + 12 \sqrt{\zeta G} + 12\zeta    \notag \\
    & \qquad + 32\max\left\{8\beta\sqrt{d\iota}, \sqrt{2\zeta}\right\}\sqrt{K + \tilde R_0 + \left|\tilde A_0\right| + 4 \hat R_0 + \left|\hat A_0\right|}    \notag \\
    & \le 352 \max\left\{64\beta^2 d\iota, 2\zeta \right\} + 30 I_c + 12 \sqrt{\zeta G} + 12\zeta + 32\max\left\{8\beta\sqrt{d\iota}, \sqrt{2\zeta}\right\}\sqrt{K}    \notag \\
    & \qquad + 32\max\left\{8\beta\sqrt{d\iota}, \sqrt{2\zeta}\right\}\sqrt{\tilde R_0 + \left|\tilde A_0\right| + 4 \hat R_0 + \left|\hat A_0\right|}. \label{eq:37}
\end{align}
By the fact that $x \leq a\sqrt{x} + b \Rightarrow x\leq 2a^2 + 2 b$, \eqref{eq:37} implies that 
\begin{align}
    & \tilde R_0 + \left|\tilde A_0\right| + 4 \hat R_0 + \left|\hat A_0\right|    \notag \\
    &  \le 896\max\left\{64\beta^2 d \iota, 2\zeta \right\}  + 60 I_c + 24 \sqrt{\zeta G} + 24\zeta + 64\max\left\{8\beta\sqrt{d\iota}, \sqrt{2\zeta}\right\}\sqrt{K}. \label{eq:sum_bound_with_G}
\end{align}
Finally, plugging \eqref{eq:sum_bound_with_G} into \eqref{eq:sum_medium} and bounding $G$ with Lemma~\ref{lm:G}, we have
\begin{align}
    & \sum_{k=1}^K \hat V_{k,1}(s_h^k)  \notag \\
    & \le \tilde R_0 + \left|\tilde A_0\right| + 4 \hat R_0 + \left|\hat A_0\right| + G  \notag \\
    & \le 896\max\left\{64\beta^2 d \iota, 2\zeta \right\}   + 24\zeta + 64\max\left\{8\beta\sqrt{d\iota}, \sqrt{2\zeta}\right\}\sqrt{K} \\
    & \qquad + 60\left(4d\iota + 4\beta \gamma^2 d\iota + 2 \beta \sqrt{d\iota} \sqrt{Md\iota + KH\alpha^2}\right)  + 24 \sqrt{\zeta Md\iota} + Md\iota \notag \\
    & \le 896\max\left\{64\beta^2 d \iota, 2\zeta \right\} + 24\zeta + 240d\iota + 240\beta\gamma^2 d \iota  + 120\beta d\iota\sqrt{M} + 24 \sqrt{\zeta Md\iota} + Md\iota\notag \\
    &\qquad + \left(64\max\left\{8\beta\sqrt{d\iota}, \sqrt{2\zeta}\right\} + 120\beta \sqrt{d\iota H \alpha^2 }\right)\sqrt{K} \notag
\end{align}
\end{proof}

\subsection{Proof of Main Results}

\begin{lemma}\label{lm:suboptimality_bound}
\textit{(Restatement of Theorem~\ref{thm:main})}
For Algorithm~\ref{alg:exp}, set $M = \log(7KH)/\log(2)$, $\{\beta_k\}_{k\ge1}$ as
\begin{align}
\beta_k = & 12 \sqrt{d \log(1 + kH/(\alpha^2 d\lambda))\log(32(\log(\gamma^2/\alpha)+1)k^2H^2/\delta)} \notag \\
& + 30 \log(32(\log(\gamma^2/\alpha) + 1)k^2H^2/\delta)/\gamma^2 + \sqrt{\lambda}B \notag,
\end{align}
and denote $\beta = \beta_K$, then for any $0<\delta'<1$, we have with probability at least $1-\delta'$, where $\delta' = 4M\delta$, after collecting $K$ trajectories, algorithm~\ref{alg:exp} returns a policy satisfying the following sub-optimality bound,
\begin{align}
& V_1^*(s_1;r) - V_1(s_1;\btheta^*, \hat\pi_r, r)  \notag \\
& \le \frac{4}{K}\left(2752\max\left\{64\beta^2 d \iota, 2\zeta \right\}  + 24\zeta + 240d\iota + 240\beta \gamma^2 d\iota + 120\beta d\iota \sqrt{M}\right)  \notag \\
&\quad + \frac{4}{\sqrt{K}}\left(64\max\left\{8\beta\sqrt{d\iota}, \sqrt{2\zeta}\right\} + 120\beta \sqrt{d\iota}\right), \notag
\end{align}
where $\iota = \log(1 + KH/(d\lambda\alpha^2))$, $\zeta = 4 \log(4\log(KH)/\delta)$. Moreover, setting $\alpha = H^{-1/2}$, $\gamma = d^{-1/4}$, and $\lambda = d/B^2$, we have the horizon-free suboptimality bound
\begin{align}
V_1^*(s_1;r) - V_1(s_1;\btheta^*, \hat\pi_r, r) = \tilde{O}\left(\frac{d^2}{K}+\frac{d}{\sqrt{K}}\right).
\end{align}
\end{lemma}

\begin{proof}[Proof of Theorem~\ref{lm:suboptimality_bound}]
The following proof is conditioned on $\cE_{\ref{lm:main_theta}}\cap\cE_{\ref{lm:A}}$, which holds with probability at least $1-4M\delta = 1-\delta'$. We have
\begin{align}
& V_1^*(s_1;r) - V_1(s_1;\btheta^*, \hat\pi_r, r)   \notag \\
& \le 4 \hat V_{K,1}\left(s_1;\btheta_K,\hat \pi_K, r_K\right)    \notag \\
& \le \frac{4}{K} \sum_{k=1}^K V_{k,1}(s_1;\btheta_k,\hat \pi_k, r_k)    \notag \\
& \le \frac{4}{K}\left(896\max\left\{64\beta^2 d \iota, 2\zeta \right\} + 24\zeta + 240d\iota + 240\beta\gamma^2 d \iota  + 120\beta d\iota\sqrt{M} + 24 \sqrt{\zeta Md\iota} + Md\iota \right)\notag \\
&\qquad + \frac{4}{\sqrt{K}}\left(64\max\left\{8\beta\sqrt{d\iota}, \sqrt{2\zeta}\right\} + 120\beta \sqrt{d\iota H \alpha^2 }\right), \notag
\end{align}
where the first inequality holds due to Lemma~\ref{lm:suboptimality_main}, the second inequality holds due to Lemma~\ref{lm:decreasing_main}, and the third equality holds due to Lemma~\ref{lm:sum}.
\end{proof}

Given the regret bound provided in Lemma~\ref{lm:suboptimality_bound}, we can prove the following sample complexity upper bound.

\begin{lemma}\label{lm:exactSC}
(Restatement of Corollary~\ref{cor:samplecomplexity})
Under the same conditions as in Theorem~\ref{lm:suboptimality_bound}, Algorithm~\ref{alg:exp} has sample complexity of
\begin{align}
    m(\varepsilon,\delta') =& \frac{16}{\varepsilon^2}\left(64\max\left\{8\beta\sqrt{d\iota}, \sqrt{2\zeta}\right\} + 120\beta \sqrt{d\iota H \alpha^2 }\right)^2 \notag \\
    &\qquad + \frac{8}{\varepsilon}\left(2752\max\left\{64\beta^2 d \iota, 2\zeta \right\}  + 24\zeta + 240d\iota + 240\beta \gamma^2 d\iota + 120\beta d\iota \sqrt{M}\right) \label{eq:app_sample_complexity}
\end{align}
Moreover, setting $\alpha = H^{-1/2}$, $\gamma = d^{-1/4}$, and $\lambda = d/B^2$, we have the horizon-free sample complexity bound
\begin{align}
    m(\varepsilon,\delta') = \tilde{O}\left(\frac{d^2}{\varepsilon^2}\right). \notag
\end{align}
\end{lemma}
\begin{proof}[Proof of Lemma~\ref{lm:exactSC}]
\eqref{eq:app_sample_complexity} is derived directly from Lemma~\ref{lm:suboptimality_bound} by setting the suboptimality to $\varepsilon$ and solving the $K$.
\end{proof}

\begin{lemma}[Restatement of Corollary~\ref{cor:samplecomplexityH}]\label{lm:exactSC_H}
When rescaling the assumption $ \sum_{h=1}^Hr_h(s_h,a_h) \le 1$ to $\sum_{h=1}^Hr_h(s_h,a_h) \le H$, under the same conditions as Lemma~\ref{lm:suboptimality_bound}, Algorithm~\ref{alg:exp} has sample complexity of
\begin{align}
    m(\varepsilon,\delta') =& \frac{16H^2}{\varepsilon^2}\left(64\max\left\{8\beta\sqrt{d\iota}, \sqrt{2\zeta}\right\} + 120\beta \sqrt{d\iota H \alpha^2 }\right)^2 \notag \\
    &\qquad + \frac{8H}{\varepsilon}\left(2752\max\left\{64\beta^2 d \iota, 2\zeta \right\}  + 24\zeta + 240d\iota + 240\beta \gamma^2 d\iota + 120\beta d\iota \sqrt{M}\right) \label{eq:app_sample_complexity_H}
\end{align}
Moreover, setting $\alpha = H^{-1/2}$, $\gamma = d^{-1/4}$, and $\lambda = d/B^2$, we have the horizon-free sample complexity bound
\begin{align}
    m(\varepsilon,\delta') = \tilde{O}\left(\frac{H^2d^2}{\varepsilon^2}\right). \notag
\end{align}
\end{lemma}
\begin{proof}[Proof of Lemma~\ref{lm:exactSC_H}]
\eqref{eq:app_sample_complexity_H} is a direct result of Lemma~\ref{lm:exactSC}. Let $r_h'(s_h,a_h) = r_h(s_h,a_h)/H$, then $\sum_{h=1}^Hr'_h(s_h,a_h) \le 1$. Thus the sample complexity of achieving an $\varepsilon/H$-optimal policy for reward $r_h'$ with probability $1-\delta'$ is 
\begin{align}
    m(\varepsilon,\delta') =& \frac{16H^2}{\varepsilon^2}\left(64\max\left\{8\beta\sqrt{d\iota}, \sqrt{2\zeta}\right\} + 120\beta \sqrt{d\iota H \alpha^2 }\right)^2 \notag \\
    &\qquad + \frac{8H}{\varepsilon}\left(2752\max\left\{64\beta^2 d \iota, 2\zeta \right\}  + 24\zeta + 240d\iota + 240\beta \gamma^2 d\iota + 120\beta d\iota \sqrt{M}\right).
\end{align}
Since $r_h(s_h,a_h) = Hr_h'(s_h,a_h)$, for the same policy, the suboptimality for rewards $r_h$ is $H$ times the suboptimality for rewards $r_h'$. Thus, the $\varepsilon/H$-optimal policy for $r_h'$ is a $\varepsilon$-optimal policy for $r_h$. Therefore, the sample complexity of achieving an $\varepsilon$-optimal policy for reward $r_h$ with probability $1-\delta'$ is $m(\varepsilon,\delta')$.
\end{proof}

\section{Proof of Lemmas in Appendix~\ref{app:main_results}}

\subsection{Proof of Lemma~\ref{lm:main_theta}}
\begin{lemma}\label{lm:bernstein}
\textit{(Theorem~4.3 in \citet{zhou22high})} Let $\left\{\mathcal{G}_k\right\}_{k=1}^{\infty}$ be a filtration, and $\left\{\mathrm{x}_k, \eta_k\right\}_{k \geq 1}$ be a stochastic process such that $\mathrm{x}_k \in \mathbb{R}^d$ is $\mathcal{G}_k$-measurable and $\eta_k \in \mathbb{R}$ is $\mathcal{G}_{k+1}$-measurable. Let $L, \sigma, \lambda, \varepsilon>0, \boldsymbol{\mu}^* \in \mathbb{R}^d$. For $k \geq 1$, let $y_k=\left\langle\boldsymbol{\mu}^*, \mathbf{x}_k\right\rangle+\eta_k$ and suppose that $\eta_k, \mathbf{x}_k$ also satisfy
\begin{align}
\mathbb{E}\left[\eta_k \mid \mathcal{G}_k\right]=0, \mathbb{E}\left[\eta_k^2 \mid \mathcal{G}_k\right] \leq \sigma^2,\left|\eta_k\right| \leq R,\left\|\mathrm{x}_k\right\|_2 \leq L    
\end{align}
For $k \geq 1$, let $\mathbf{Z}_k=\lambda \mathbf{I}+\sum_{i=1}^k \mathbf{x}_i \mathbf{x}_i^{\top}, \mathbf{b}_k=\sum_{i=1}^k y_i \mathbf{x}_i, \boldsymbol{\mu}_k=\mathbf{Z}_k^{-1} \mathbf{b}_k$, and

\begin{align}
\beta_k=& 12 \sqrt{\sigma^2 d \log \left(1+k L^2 /(d \lambda)\right) \log \left(32(\log (R / \varepsilon)+1) k^2 / \delta\right)}  \notag \\
&+24 \log \left(32(\log (R / \varepsilon)+1) k^2 / \delta\right) \max _{1 \leq i \leq k}\left\{\left|\eta_i\right| \min \left\{1,\left\|\mathbf{x}_i\right\|_{\mathbf{Z}_{i-1}^{-1}}\right\}\right\}+6 \log \left(32(\log (R / \varepsilon)+1) k^2 / \delta\right) \varepsilon .
\end{align}
Then, for any $0<\delta<1$, we have with probability at least $1-\delta$ that,
\begin{align}
\forall k \geq 1,\left\|\sum_{i=1}^k \mathbf{x}_i \eta_i\right\|_{\mathbf{Z}_k^{-1}} \leq \beta_k,\quad \left\|\boldsymbol{\mu}_k-\boldsymbol{\mu}^*\right\|_{\mathbf{Z}_k} \leq \beta_k+\sqrt{\lambda}\left\|\boldsymbol{\mu}^*\right\|_2 \notag 
\end{align}
\end{lemma}

\begin{lemma}\label{lm:var}
Let $\tilde V_{k,h}$, $\hat V_{k,h}$, $\dot{\tilde\bSigma}_{k,m}$, $\dot{\hat\bSigma}_{k,m}$, $\tilde\btheta_{k,m}$, $\hat\btheta_{k,m}$, $\tilde \bphi_{k,h,m}$, $\hat \bphi_{k,h,m}$ be defined in Algorithm~\ref{alg:exp}, for any $k\in[K]$, $h\in[H]$, $m\in\overline{[M]}$. We have

    \begin{align}
        & \left|\VV \hat V_{k,h+1}^{2^m}\left(s_h^k, a_h^k\right) - \hat{\VV} \hat V_{k,h+1}^{2^m}\left(s_h^k, a_h^k\right)\right|  \notag \\
        & \le \min\left\{1, \left\| \hat \bphi_{k,h,m+1} \right\|_{\dot{\hat\bSigma}_{k,m+1}^{-1}} \left\| \hat \btheta_{k,m+1} - \btheta^* \right\|_{\dot{\hat\bSigma}_{k,m+1}} \right\}  \notag \\
        & \qquad + \min\left\{1, 2\left\| \hat \bphi_{k,h,m} \right\|_{\dot{\hat\bSigma}_{k,m}^{-1}} \left\| \hat \btheta_{k,m} - \btheta^* \right\|_{\dot{\hat\bSigma}_{k,m}} \right\}, \label{eq:var_hat}
    \end{align}
    and
    \begin{align}
        & \left|\VV \tilde V_{k,h+1}^{2^m}\left(s_h^k, a_h^k\right) - \tilde{\VV} \tilde V_{k,h+1}^{2^m}\left(s_h^k, a_h^k\right)\right|  \notag \\
        & \le \min\left\{1, \left\| \tilde \bphi_{k,h,m+1} \right\|_{\dot{\tilde\bSigma}_{k,m+1}^{-1}} \left\| \tilde \btheta_{k,m+1} - \btheta^* \right\|_{\dot{\tilde\bSigma}_{k,m+1}} \right\}  \notag \\
        & \qquad + \min\left\{1, 2\left\| \tilde \bphi_{k,h,m} \right\|_{\dot{\tilde\bSigma}_{k,m}^{-1}} \left\| \tilde \btheta_{k,m} - \btheta^* \right\|_{\dot{\tilde\bSigma}_{k,m}} \right\}. \label{eq:var_tilde}
    \end{align}
\end{lemma}
 \begin{proof}[Proof of Lemma~\ref{lm:var}]
     The proof follows the proof of Lemma C.1 in \cite{zhou2021nearly}. We first prove \eqref {eq:var_hat}, and the proof of \eqref{eq:var_tilde} is similar. We have
     \begin{align}
         & |  [\hat \VV_{k,h}\hat V^{2^m}_{k,h+1}](s_h^k,a_h^k) - [\VV_{k,h}\hat V_{k,h+1}](s_h^k,a_h^k)  | \notag \\
         & = | [\la \hat\bphi_{k,h,m+1}, \hat\btheta_{k,m+1} \ra]_{[0 ,1]} -  \la \hat \bphi_{k,h,m+1}, \btheta^* \ra \notag\\
         & \qquad  + (\la \hat \bphi_{k,h,m}, \btheta^* \ra)^2 - [\la \hat \bphi_{k,h,m}, \hat\btheta_{k,m} \ra]^2_{[0 ,1]} | \notag\\
         & \le \underbrace{| [\la \hat \bphi_{k,h,m+1}, \hat\btheta_{k,m+1} \ra]_{[0 ,1]} -  \la \hat \bphi_{k,h,m+1}, \btheta^* \ra|}_{I_1} \notag\\
         & \qquad  + \underbrace{|(\la \hat \bphi_{k,h,m}, \btheta^* \ra)^2 - [\la \hat \bphi_{k,h,m}, \hat\btheta_{k,m} \ra]^2_{[0 ,1]} |}_{I_2} \label{eq:var_I1I2}
     \end{align}
where the inequality holds due to triangle inequality. We have $I_1 \le 1$ since both terms in $I_1$ lie in the interval $[0,1]$. Furthermore, 
\begin{align}
    I_1 &\le | [\la \hat \bphi_{k,h,m+1}, \hat\btheta_{k,m+1} \ra] -  \la \hat \bphi_{k,h,m+1}, \btheta^* \ra| \notag\\
        & = | [\la \hat \bphi_{k,h,m+1}, \hat\btheta_{k,m+1} - \btheta^* \ra | \notag\\
        & \le \| \bphi_{k,h,m+1} \|_{\dot{\hat\bSigma}_{k,m+1}^{-1}} \| \hat\btheta_{k,m+1} - \btheta^* \|_{\dot{\hat\bSigma}_{k,m+1}},\notag
\end{align}
where the first inequality holds due to $\la \hat \bphi_{k,h,m+1}(s_h^k,a_h^k), \btheta^* \ra \in [0,1] $, the second inequality holds due to Cauchy-Schwarz inequality. Thus, we obtain
\begin{align}
    I_1 \le \min \{1,  \| \bphi_{k,h,m+1} \|_{\dot{\hat\bSigma}_{k,m+1}^{-1}} \| \hat\btheta_{k,m+1} - \btheta^* \|_{\dot{\hat\bSigma}_{k,m+1}} \} \label{eq:var_I1}
\end{align}
For $I_2$, we have
\begin{align}
I_2 &= \left|(\la \hat \bphi_{k,h,m}(s_h^k,a_h^k), \btheta^* \ra) - [\la \hat \bphi_{k,h,m}, \hat\btheta_{k,m} \ra]_{[0 ,1]}\right| \cdot \left|(\la \hat \bphi_{k,h,m}(s_h^k,a_h^k), \btheta^* \ra) + [\la \hat \bphi_{k,h,m}, \hat\btheta_{k,m} \ra]_{[0 ,1]}\right| \notag\\
& \le 2 \left|(\la \hat \bphi_{k,h,m}(s_h^k,a_h^k), \btheta^* \ra) - \la \hat \bphi_{k,h,m}, \hat\btheta_{k,m}, \ra\right| \notag\\
& \le 2 \| \bphi_{k,h,m}(s_h^k,a_h^k) \|_{\dot{\hat\bSigma}_{k,m}^{-1}} \| \hat\btheta_{k,m} - \btheta^*   \|_{\dot{\hat\bSigma}_{k,m}} \notag
\end{align}
where the first inequality holds due to that both $ \la \hat \bphi_{k,h,m}(s_h^k,a_h^k), \btheta^* \ra$ and $[\la \hat \bphi_{k,h,m}, \hat\btheta_{k,m} \ra]_{[0 ,1]} $ lie in the interval $[0,1]$, and the second inequality holds due to Cauchy-Schwarz inequality.
Since $I_2$ belongs to the interval $[0,1]$, we have
\begin{align}
I_2 \le \min \{1,   2 \| \bphi_{k,h,m}(s_h^k,a_h^k) \|_{\dot{\hat\bSigma}_{k,m}^{-1}} \| \hat\btheta_{k,m} - \btheta^*   \|_{\dot{\hat\bSigma}_{k,m}}\} \label{eq:var_I2}
\end{align}
Substituting \eqref{eq:var_I1} and \eqref{eq:var_I2} into \eqref{eq:var_I1I2}, we obtain \eqref{eq:var_hat}. The proof of \ref{eq:var_tilde} is nearly identical to the proof of \eqref{eq:var_I1I2}. The only difference is to replace $\hat \bphi$ with $\tilde \bphi$, $\hat\btheta$ with $\tilde\btheta$, $\dot{\hat\bSigma}$ with $\dot{\tilde\bSigma}$.
 \end{proof}
 
\begin{proof}[Proof of Lemma~\ref{lm:main_theta}]
     The proof follows Lemma C.2 in \citet{zhou22high}. Symbols we used here may have small intuitively understandable modification compared to Algorithm~\ref{alg:home} since we have to distinguish between Algorithm~\ref{alg:home} applied to $\tilde{V}_{k,h}$ and $\hat{V}_{k,h}$. We first prove that Equation \eqref{eq:hattheta} holds with high probability. By definitions,
     \begin{align}
         &\hat \sigma_{k,h,m}^2 = \max\left\{ \gamma^2\left\|\hat\bphi_{k,h,m}\right\|_{\hat\bSigma_{k,h,m}^{-1}}, \left[\hat\VV_{k,m}\hat V_{k, h+1}^{2^m}\right](s_h^k, a_h^k) + \hat E_{k,h,m}, \alpha^2\right\} \notag\\
         &\hat \sigma_{k,h,m}^2 = \max\left\{ \gamma^2\left\|\hat \bphi_{k,h,M-1}\right\|_{\hat\bSigma_{k,h,M-1}^{-1}}, 1,   \alpha^2 \right\}. \notag
     \end{align}
     We define $\cC_{k,m}$ as 
     \begin{align}
         \hat\cC_{k,m} := \{\btheta:\| \btheta - \hat \btheta_{k,m} \|_{\dot{\hat \bSigma}_{k,m}} \le \beta_k \}. \notag
     \end{align}
     For each $m$, let 
     \begin{align}
         \xb_{k,h,m} & = \hat\sigma^{-1}_{k,h,m} \hat \bphi_{k,h,m}, \notag\\
         \eta_{k,h,m} & = \hat\sigma^{-1}_{k,h,m} \ind\{\btheta^*\in \hat\cC_{k,m} \cap \hat\cC_{k,m+1}  \} [ \hat V_{k,h+1}^{2^m}(s_{h+1}^k) - \la \hat \bphi_{k,h,m}, \btheta^* \ra ], \notag\\
         \eta_{k,h,M-1} & = \hat\sigma^{-1}_{k,h,M-1}  [ \hat V_{k,h+1}^{2^{M-1}} - \la \hat \bphi_{k,h,M-1}, \btheta^* \ra ], \notag\\
         \cG_{k,h} &= \cF_{k,h}, \notag\\
         \bmu^* &= \btheta^*. \notag
     \end{align}
     We have 
     \begin{align}
         \EE[\eta_{k,h,m}|\cG_{k,h}] = 0, \quad \| \xb_{k,h,m} \|_2 \le \hat \sigma^{-1}_{k,h,m} \le 1/\alpha, \quad | \eta_{k,h,m} | \le 1/\alpha \notag
     \end{align}
     Since $\ind\{\btheta^*\in \hat\cC_{k,m} \cap \hat\cC_{k,m+1}\}$ is $\cG_{k,h}$-measurable, then we can bound the variance for $m\in\overline{[M]}$ as follows:
    \begin{align}
    \EE[\eta_{k,h,m}^2 |\cG_{k,h}] & = \hat\sigma^{-2}_{k,h,m}\ind\{\btheta^*\in \hat\cC_{k,m} \cap \hat\cC_{k,m+1}\} [\VV \hat V^{2^m}_{k,h+1}](s_h^k, a_h^k) \notag\\
    & \le \hat\sigma^{-2}_{k,h,m}\ind\{\btheta^*\in \hat\cC_{k,m} \cap \hat\cC_{k,m+1}\} \Bigg[\hat{\VV} \hat V_{k,h+1}^{2^m}\left(s_h^k, a_h^k\right)\notag \\
    &\qquad + \min\left\{1, \left\| \hat \bphi_{k,h,m+1} \right\|_{\dot{\hat\bSigma}_{k,m+1}^{-1}} \left\| \hat \btheta_{k,m+1} - \btheta^* \right\|_{\dot{\hat\bSigma}_{k,m+1}} \right\} \notag\\
    &\qquad + \min\left\{1, 2\left\| \hat \bphi_{k,h,m} \right\|_{\dot{\hat\bSigma}_{k,m}^{-1}} \left\| \hat \btheta_{k,m} - \btheta^* \right\|_{\dot{\hat\bSigma}_{k,m}} \right\}\Bigg] \notag\\
    &  \le \hat\sigma^{-2}_{k,h,m} \Bigg[ \hat{\VV} \hat V_{k,h+1}^{2^m}\left(s_h^k, a_h^k\right) + \min\left\{1, \beta_k\left\| \hat \bphi_{k,h,m+1} \right\|_{\dot{\hat\bSigma}_{k,m+1}^{-1}}  \right\} \notag\\
    & \qquad + \min\left\{1, 2\beta_k\left\| \hat \bphi_{k,h,m} \right\|_{\dot{\hat\bSigma}_{k,m}^{-1}}  \right\} \Bigg] \notag\\
    & \le 1, \notag
    \end{align}
    where the first inequality holds due to Lemma~\ref{lm:var}, the second inequality holds due to the definition of the indicator function, and the third inequality holds due to the definition of $ \hat\sigma_{k,h,m}^{-2} $. For $m=M-1$, we have $\EE[\eta_{k,h,m}^2 |\cG_{k,h}] \le 1$ directly by the definition of $\hat\sigma^{2}_{k,h,m}$. For any $m\in \overline{[M]}$, we have
    \begin{align}
    |\eta_{k,h,m}| \max \{ 1, \| \xb_{k,h,m} \|_{\hat\bSigma^{-1}_{k,h-1,m}} \} \le \hat\sigma_{k,h,m}^{-2} \| \hat\bphi_{k,h,m} \|_{\hat\bSigma^{-1}_{k,h-1,m}} \le 1/\gamma^2,\notag
    \end{align}
    where the first inequality follows from the definition of $\eta_{k,h,m}$ and $\xb_{k,h,m}$, and the second inequality follows from the definition of $\hat\sigma_{k,h,m} $. Let
    \begin{align}
        y_{k,h,m} &= \la \bmu^*,\xb_{x,h,m} \ra + \eta_{k,h,m}, \notag\\
        \Zb_{k,m} &= \lambda\Ib + \sum_{i=1}^{k}\sum_{h=1}^H \xb_{i,h,m}\xb_{i,h,m}^\top = \dot{\hat\bSigma}_{k+1,m}, \notag\\
        \bb_{k,m} &= \sum_{i=1}^{k}\sum_{h=1}^H \xb_{i,h,m}y_{i,h,m}, \notag\\
        \bmu_{k,m} &= \Zb_{k,m}^{-1}\bb_{k,m},\notag\\
        \varepsilon &= 1 / \gamma^2.\notag
    \end{align}
    Then, by Lemma~\ref{lm:bernstein}, for each $m\in\overline{[M]}$, with probability at least $1-\delta$, $\forall k \in [K+1]$,
    \begin{align}
    \|\bmu_{k-1,m} - \btheta^*\|_{\dot{\hat\bSigma}_{k,m}} \le & 12 \sqrt{d \log(1 + kH/(\alpha^2 d\lambda))\log(32(\log(\gamma^2/\alpha)+1)k^2H^2/\delta)} \notag \\
& + 30 \log(32(\log(\gamma^2/\alpha) + 1)k^2H^2/\delta)/\gamma^2 + \sqrt{\lambda}B \notag \\
 & = \beta_k \label{eq:high_probability}
    \end{align}
    Define the event that \eqref{eq:high_probability} happens for all $k$ and $m$ as $\hat\cE$. Conditioned on $\hat\cE$, the following properties hold:
    \begin{itemize}
        \item For $k=1$, $m \in \overline{[M]}$, by the definition of $\hat\btheta_{1,m}$ and $\dot{\hat\bSigma}_{1,m}$, we have $ \|\btheta^* - \hat\btheta_{1,m}\|_{\dot{\hat\bSigma}_{1,m}} = \| \btheta^* \|_{\lambda\Ib} \le \sqrt{\lambda} B = \beta_1$, which implies
        \begin{align}
        \btheta^* \in \hat \cC_{1,m} \label{eq:initial_condition1} 
        \end{align}
        \item For $k\in [K]$ and $m = M-1$, we directly have $\bmu_{k,M-1}=\hat\btheta_{k+1,M-1}$, which implies
        \begin{align}
        \btheta^* \in \hat \cC_{k+1,M-1}. \label{eq:initial_condition2}
        \end{align}
        \item For $k\in[K]$ and $m\in\overline{[M-1]}$, we have
        \begin{align}
            \btheta^* \in \hat \cC_{k,m}\cap\hat\cC_{k,m+1} \Rightarrow y_{k,h,m} = \hat \sigma^{-1} \hat V_{k,h+1}^{2^m}(s_{h+1}^k) \Rightarrow \bmu_{k,m} = \hat\btheta_{k+1,m} \Rightarrow \btheta^*\in\hat\cC_{k+1,m}. \label{eq:induction_rule}
        \end{align}
        Therefore, by induction based on initial conditions \eqref{eq:initial_condition1} and \eqref{eq:initial_condition2}, induction rule \eqref{eq:induction_rule}, we have for $k\in[K]$ and $m\in[M]$, $\btheta^*\in\hat \cC_{k,m}$. Taking the union bound gives that \eqref{eq:hattheta} happens with probability at least $1-M\delta$. 
We can use the nearly identical argument to prove that \eqref{eq:tildetheta} holds with probability at least $1-M\delta$. The only difference is to replace $\hat\sigma$ with $\tilde\sigma$, $\hat\bphi$ with $\tilde \bphi$, $\hat\VV$ with $\tilde\VV$, $\hat V$ with $\tilde \VV$, $\hat\bSigma$ with $\tilde\bSigma$, $\dot{\hat\bSigma}$ with $\dot{\tilde\bSigma}$, $\hat\btheta$ with $\tilde\btheta$. By taking the union bound, we obtain that with probability at least $1-2M\delta$, Equations~\eqref{eq:hattheta}~\eqref{eq:tildetheta} both hold.
    \end{itemize}
     For \eqref{eq:theta_hat} and \eqref{eq:theta_tilde}, we have
\begin{align}
    \left\| \btheta_k - \btheta^* \right\|_{ \dot{\hat\bSigma}_{k,0}} \le  \left\| \btheta_k - \hat \btheta_{k,m} \right\|_{ \dot{\hat\bSigma}_{k,0}} +  \left\| \hat \btheta_{k,m} - \btheta^* \right\|_{\dot{\hat\bSigma}_{k,0}} \leq 2 \beta_k, \notag \\
    \left\| \btheta_k - \btheta^* \right\|_{ \dot{\tilde\bSigma}_{k,0}} \le  \left\| \btheta_k - \tilde \btheta_{k,m} \right\|_{ \dot{\tilde\bSigma}_{k,0}} +  \left\| \tilde \btheta_{k,m} - \btheta^* \right\|_{\dot{\tilde\bSigma}_{k,0}} \leq 2 \beta_k \notag
\end{align}
\end{proof}

\subsection{Proof of Lemma~\ref{lm:traj}}
\begin{proof}[Proof of Lemma~\ref{lm:traj}]
    We prove this inequality by induction. Suppose 
    \begin{align}
        V_{h+1}(s_{h+1};\btheta_K,\pi,r) - &V_{h+1}(s_{h+1};\btheta^*,\pi,r) \notag \\
        & =   \EE_{\texttt{traj}\sim(\pi,\PP)|\texttt{traj}_{h+1}} W_{h+1}( \{ (\PP_K - \PP) V_{h+1}(s_{h},a_{h};\btheta_K,\pi,r) \}), \label{eq:traj_condition1}
    \end{align}
    which is true for $h=H$. Then, we have
    \begin{align}
        &V_{h}(s_{h};\btheta_K,\pi,r) - V_{h}(s_{h};\btheta^*,\pi,r)  \notag \\
        & = \min\big\{1, r_h(s_h,a_h) + \PP_K V_{h+1}(s_h,a_h;\btheta_K,\pi,r) - \big(r_h(s_h,a_h) + \PP V_{h+1}(s_h,a_h;\btheta^*,\pi,r)\big) \big\} \notag \\
        & = \min\big\{1,\PP_K V_{h+1}(s_h,a_h;\btheta_K,\pi,r) - \PP V_{h+1}(s_h,a_h;\btheta^*,\pi,r) \big\} \notag \\
        & = \min\{1,(\PP_K- \PP) V_{h+1}(s_h,a_h;\btheta_K,\pi,r) + \PP (V_{h+1}(s_h,a_h;\btheta_K,\pi,r) - V_{h+1}(s_h,a_h;\btheta^*,\pi,r) )\}  \notag \\
        & = \min\{1, (\PP_K- \PP) V_{h+1}(s_h,a_h;\btheta_K,\pi,r) \notag\\
        &\quad+ \EE_{s_{h+1}\sim\PP(\cdot|s_h,a_h)}\EE_{\texttt{traj}\sim(\pi,\PP)|\texttt{traj}_{h+1}} W_{h+1}( \{ (\PP_K - \PP) V_{h+1}(s_h;\btheta_K,\pi,r) \}) \}  \notag \\
        & = \EE_{\texttt{traj}\sim(\pi,\PP)|\texttt{traj}_{h}} \min\{1, (\PP_K- \PP) V_{h+1}(s_h,a_h;\btheta_K,\pi,r) + W_{h+1}( \{ (\PP_k - \PP) V_{h+1}(s_h,a_h;\btheta_K,\pi,r) \}) \}  \notag \\
        & = \EE_{\texttt{traj}\sim(\pi,\PP)|\texttt{traj}_{h}} W_{h}( \{ (\PP_k - \PP) V_{h+1}(s_h,a_h;\btheta_K,\pi,r) \}). \notag
    \end{align}
The first equality holds due to that $V_{h}(s_{h};\btheta_K,\pi,r)$ and $V_{h}(s_{h};\btheta^*,\pi,r)$ both belong to $[0,1]$, the third equality holds due to \eqref{eq:traj_condition1}, and the forth equality holds due to that $\EE_{\texttt{traj}\sim(\pi,\PP)|\texttt{traj}_{h}} = \EE_{s_{h+1}\sim\PP(\cdot|s_h,a_h)}\EE_{\texttt{traj}\sim(\pi,\PP)|\texttt{traj}_{h+1}}$. Thus, by induction, we obtain the desired result \eqref{eq:traj}.
\end{proof}

\subsection{Proof of Lemma~\ref{lm:optimism}}

\begin{proof}[Proof of Lemma~\ref{lm:optimism}]

We first prove \eqref{eq:optimism_medium} by induction.
\begin{align}
\EE_{\texttt{traj}\sim(\pi,\PP)|\texttt{traj}_1} W_1\left(\left\{  u_{K,h}(s_h,\pi(s_h);\btheta_K, \pi, r) \right\}\right) \le \hat V_{K,1}\left(s_1;\btheta_K,\pi,r\right). \label{eq:optimism_medium}
\end{align}
Suppose
\begin{align}
\EE_{\texttt{traj}\sim(\pi,\PP)|\texttt{traj}_{h+1}} W_{h+1}(\{ u_{K,h}(s_h,\pi(s_h);\btheta_K, \pi, r) \}) \le \hat V_{K,h+1}(s_1;\btheta_K,\pi,r), \label{eq:optimism_condition1}
\end{align}
which is true for $h=H$. Then, 
    \begin{align}
        &\hat V_{K,h}(s_h;\btheta_K,\pi,r) - \EE_{\texttt{traj}\sim(\pi,\PP)|\texttt{traj}_h} W_h(\{ u_{K,h}(s_h,\pi(a_h);\btheta_K, \pi, r) \})  \notag \\
        &\ge \min\left\{0, u_{K,h}(s_h,\pi(s_h);\btheta_K,\pi, r) + 2 \beta  \left\| \bphi_{\hat V_{K,h+1}(\cdot;\btheta_K,\pi,r)} (s_h,\pi(s_h))\right\|_{\dot{\hat\bSigma}_{K,0}}  \right. \notag \\
         &\left.\qquad +\bphi^\top_{\hat V_{K,h+1}(\cdot;\btheta_K,\pi,r)} \left(s_h,\pi(s_h)\right)\btheta_K - \EE_{\texttt{traj}\sim(\pi,\PP)|\texttt{traj}_h} W_h(\{ u_{K,h}(s_h,\pi(s_h);\btheta_K, \pi, r))\}\right\} \notag \\
         & \ge \min\bigg\{0, u_{K,h}(s_h,\pi(s_h);\btheta_K,\pi, r) + 2 \beta  \left\| \bphi_{\hat V_{K,h+1}(\cdot;\btheta_K,\pi,r)} (s_h,\pi(s_h))\right\|_{\dot{\hat\bSigma}_{K,0}} \notag \\
         & \qquad + \bphi^\top_{\hat V_{K,h+1}(\cdot;\btheta_K,\pi,r)} (s_h,\pi(s_h))\btheta_K - u_{K,h}(s_h,\pi(s_h);\btheta_K,\pi, r)   \notag \\
         & \qquad - \EE_{s_{h+1}\sim\PP(\cdot|s_h,\pi(s_h))}\EE_{\texttt{traj}\sim(\pi,\PP)|\texttt{traj}_{h+1}} W_{h+1}(\{ u_{K,h}(s_h,\pi(s_h); \btheta_K, \pi, r) ) \bigg\} \notag \\
         & \ge \min\bigg\{0,  2 \beta  \left\| \bphi_{\hat V_{K,h+1}(\cdot;\btheta_K,\pi,r)} (s_h,\pi(s_h))\right\|_{\dot{\hat\bSigma}_{K,0}} + \bphi^\top_{\hat V_{K,h+1}(\cdot;\btheta_K,\pi,r)} (s_h,\pi(s_h))\btheta_K   \notag \\
         & \left.\qquad -   \EE_{s_{h+1}\sim\PP(\cdot|s_h,\pi(s_h))} \hat V_{K,h+1}(s_{h+1};\btheta_K,\pi,r) \right\} \notag \\
         & \geq \min\left\{0, 2 \beta  \left\| \bphi_{\hat V_{K,h+1}(\cdot;\btheta_K,\pi,r)} (s_h,\pi(s_h))\right\|_{\dot{\hat\bSigma}_{K,0}} + \bphi^\top_{\hat V_{K,h+1}(\cdot;\btheta_K,\pi,r)} (s_h,\pi(s_h))(\btheta_K - \btheta^*)\right\} \notag \\
         & \geq \min\left\{0, 2 \beta  \left\| \bphi_{\hat V_{K,h+1}(\cdot;\btheta_K,\pi,r)} (s_h,\pi(s_h))\right\|_{\dot{\hat\bSigma}_{K,0}} - 2 \beta  \left\| \bphi_{\hat V_{K,h+1}(\cdot;\btheta_K,\pi,r)} (s_h,\pi(s_h))\right\|_{\dot{\hat\bSigma}_{K,0}}\right\}  \notag \\
         & \geq 0, \notag
    \end{align}
where the first inequality holds due to the definition of $\hat V_{K,h}$, the second inequality holds due to the definition of $W_h(\cdot)$ and $\EE_{\texttt{traj}\sim(\pi,\PP)|\texttt{traj}_{h}} = \EE_{s_{h+1}\sim\PP(\cdot|s_h,\pi(s_h))}\EE_{\texttt{traj}\sim(\pi,\PP)|\texttt{traj}_{h+1}}$, the third inequality holds due to \ref{eq:optimism_condition1}, the fifth inequality holds due to  Lemma~\ref{lm:main_theta}. Thus, by induction, \ref{eq:optimism_medium} holds. Thanks to the optimism of $\hat V_{K,1}(s_1;\btheta_K,\pi_K,r_K)$, we have
\begin{align}
\hat V_{K,1}(s_1;\btheta_K,\pi,r) \le \hat V_{K,1}(s_1;\btheta_K,\pi_K,r_K), \notag
\end{align}
which concludes the proof.
\end{proof}

\subsection{Proof of Lemma~\ref{lm:ucb}}
\begin{proof}[Proof of Lemma~\ref{lm:ucb}]
For the equation \eqref{eq:hatucb}, we have
\begin{align}
&\hat V_{k,h}(s_h^k) - u_{k,h}(s_h^k, a_h^k) - \PP \hat V_{k,h+1}(s_h^k,a_h^k)  \notag \\
& \le \min\left\{ 1 , 2 \beta  \left\| \hat \bphi_{k,h,0} (s_h^k,a_h^k)\right\|_{\dot{\hat\bSigma}_{k,0}^{-1}} + \hat\bphi^\top_{k,h,0} (s_h^k,a_h^k)\btheta_k - \hat\bphi_{k,h,0}^\top(s_h^k,a_h^k) \btheta \right\}  \notag \\
& = \min\left\{ 1 , 2 \beta  \left\| \hat\bphi_{k,h,0} (s_h^k,a_h^k)\right\|_{\dot{\hat\bSigma}_{k,0}^{-1}} + \hat\bphi^\top_{k,h,0} (s_h^k,a_h^k)(\btheta_k-\btheta)  \right\} \notag \\
& \le \min\left\{ 1 , 2 \beta  \left\| \hat\bphi_{k,h,0} (s_h^k,a_h^k)\right\|_{\dot{\hat\bSigma}_{k,0}^{-1}} + \left\|\hat\bphi^\top_{k,h,0} (s_h^k,a_h^k) \right\|_{\dot{\hat\bSigma}_{k,0}^{-1}} \left\|\btheta_k-\btheta\right\|_{\dot{\hat\bSigma}_{k,0}}\right\}  \notag \\
& \le \min\left\{ 1 , 4 \beta  \left\| \hat \bphi_{k,h,0} (s_h^k,a_h^k)\right\|_{\dot{\hat\bSigma}_{k,0}^{-1}}  \right\} \notag \\
& \le 4 \min\left\{ 1 ,  \beta  \left\| \hat \bphi_{k,h,0} (s_h^k,a_h^k)\right\|_{\dot{\hat\bSigma}_{k,0}^{-1}}   \right\} \notag
\end{align}
where the first inequality holds due to that each term lies in the interval $[0,1]$, the second inequality holds due to Cauchy-Schwartz inequality, and the third inequality holds due to lemma~\ref{lm:main_theta}. For the equation \eqref{eq:tildeucb}, we have
\begin{align}
& \tilde V_{k,h}(s_h^k) - r_{k,h}(s_h^k,a_h^k) - \PP \tilde V_{k,h+1}(s_h^k,a_h^k)  \notag \\
&\le \min\left\{1,  \tilde\bphi^\top_{k,h,0} (s_h^k,a_h^k)\btheta_k - \tilde \bphi_{k,h,0}^\top(s_h^k,a_h^k) \btheta \right\}  \notag \\
& = \min\left\{1,  \tilde\bphi^\top_{k,h,0} (s_h^k,a_h^k)(\btheta_k - \btheta)\right\}  \notag \\
& \le \min\left\{1,  \left\| \tilde\bphi^\top_{k,h,0} (s_h^k,a_h^k) \right\|_{\dot{\tilde\bSigma}_{k,0}^{-1}} \left\|\btheta_k-\btheta\right\|_{\dot{\tilde\bSigma}_{k,0}} \right\}  \notag \\
& \le \min\left\{1, 2\beta\left\|\tilde\bphi^\top_{k,h,0} (s_h^k,a_h^k) \right\|_{\dot{\tilde\bSigma}_{k,0}^{-1}} \right\} \notag \\
& \le 2 \min\left\{1, \beta\left\|\tilde\bphi^\top_{k,h,0} (s_h^k,a_h^k) \right\|_{\dot{\tilde\bSigma}_{k,0}^{-1}} \right\}, \notag
\end{align}
where the first inequality holds due to that each term lies in the interval $[0,1]$, the second inequality holds due to the Cauchy-Schwartz inequality, and the third inequality holds due to Lemma~\ref{lm:main_theta}.
\end{proof}

\subsection{Proof of Lemma~\ref{lm:R}}

\begin{lemma}[Lemma~B.1, \citet{zhou22high}]\label{lm:ub}
Let $\left\{\sigma_k, \beta_k\right\}_{k \geq 1}$ be a sequence of non-negative numbers, $\alpha, \gamma>0,\left\{\mathbf{x}_k\right\}_{k \geq 1} \subset \mathbb{R}^d$ and $\left\|\mathbf{x}_k\right\|_2 \leq L$. Let $\left\{\mathbf{Z}_k\right\}_{k \geq 1}$ and $\left\{\bar{\sigma}_k\right\}_{k \geq 1}$ be recursively defined as follows: $\mathbf{Z}_1=\lambda \mathbf{I}$
\begin{align}
 \forall k \geq 1, \bar{\sigma}_k=\max \left\{\sigma_k, \alpha, \gamma\left\|\mathbf{x}_k\right\|_{\mathbf{z}_k^{-1}}^{1 / 2}\right\}, \mathbf{Z}_{k+1}=\mathbf{Z}_k+\mathbf{x}_k \mathbf{x}_k^{\top} / \bar{\sigma}_k^2 .   \notag
\end{align}

Let $\iota=\log \left(1+K L^2 /\left(d \lambda \alpha^2\right)\right)$. Then we have
\begin{align}
 \sum_{k=1}^K \min \left\{1, \beta_k\left\|\mathbf{x}_k\right\|_{\mathbf{Z}_k^{-1}}\right\} \leq 2 d \iota+2 \max _{k \in[K]} \beta_k \gamma^2 d \iota+2 \sqrt{d \iota} \sqrt{\sum_{k=1}^K \beta_k^2\left(\sigma_k^2+\alpha^2\right)} . \notag  
\end{align}
\end{lemma}

\begin{proof}[Proof of Lemma~\ref{lm:R}]
The proof is nearly identical to the proof of Lemma C.5 in \citet{zhou22high}. The only difference is to replace $\hat\bSigma_{k,m}$ with $\dot{\hat\bSigma}_{k,m}$ (or $\dot{\tilde\bSigma}_{k,m}$), $\tilde\bSigma_{k,h,m}$ with $\hat\bSigma_{k,h,m}$ (or still $\tilde\bSigma_{k,h,m}$), $\bphi_{k,h,m}$ with $\hat\bphi_{k,h,m}$ (or $\tilde\bphi_{k,h,m}$).
\end{proof}

\subsection{Proof of Lemma~\ref{lm:S}}

\begin{proof}[Proof of Lemma~\ref{lm:S}]
The proof follows the proof of Lemma 25 in \citet{zhang2021improved} and Lemma C.6 in \citet{zhou22high}. We have,
\begin{align}
\hat S_m=&\sum_{k=1}^K \sum_{h=1}^H I_h^k\left[\left[\PP \hat V_{k, h+1}^{2^{m+1}}\right]\left(s_h^k, a_h^k\right)-\left(\left[\PP \hat V_{k, h+1}^{2^m}\right]\left(s_h^k, a_h^k\right)\right)^2\right]   \notag \\
=&\sum_{k=1}^K \sum_{h=1}^H I_h^k\left[\left[\PP \hat V_{k, h+1}^{2^{m+1}}\right]\left(s_h^k, a_h^k\right)-\hat V_{k, h+1}^{2^{m+1}}\left(s_{h+1}^k\right)\right]+\sum_{k=1}^K \sum_{h=1}^HI_h^k\bigg[\hat V_{k, h}^{2^{m+1}}\left(s_h^k\right)   \notag \\
& -\left(\left[\PP \hat V_{k, h+1}^{2^m}\right]\left(s_h^k, a_h^k\right)\right)^2\bigg] +\sum_{k=1}^K \sum_{h=1}^H I_h^k\left(\hat V_{k, h+1}^{2^{m+1}}\left(s_{h+1}^k\right)- \hat V_{k, h}^{2^{m+1}}\left(s_h^k\right)\right)   \notag \\
\leq & \hat A_{m+1}+\sum_{k=1}^K \sum_{h=1}^H I_h^k\left[\hat V_{k, h}^{2^{m+1}}\left(s_h^k\right)-\left(\left[\PP\hat V_{k, h+1}^{2^m}\right]\left(s_h^k, a_h^k\right)\right)^2\right]+\sum_{k=1}^K I_{h_k}^k \hat V_{k, h_k+1}^{2^{m+1}}\left(s_{h_k+1}^k\right),\label{eq:18}
\end{align}
where $h_k$ is the largest index satisfying $I_h^k = 1$. For the second term, we have
\begin{align}
&\sum_{k=1}^K \sum_{h=1}^H I_h^k\left[\hat V_{k, h}^{2^{m+1}}\left(s_h^k\right)-\left(\left[\mathbb{P} \hat V_{k, h+1}^{2^m}\right]\left(s_h^k, a_h^k\right)\right)^2\right]    \notag \\
&\leq \sum_{k=1}^K \sum_{h=1}^H I_h^k\left[\hat V_{k, h}^{2^{m+1}}\left(s_h^k\right)-\left(\left[\mathbb{P} \hat V_{k, h+1}\right]\left(s_h^k, a_h^k\right)\right)^{2^{m+1}}\right]   \notag \\
&=\sum_{k=1}^K \sum_{h=1}^H I_h^k\left(\hat V_{k, h}\left(s_h^k\right)-\left[\mathbb{P} \hat V_{k, h+1}\right]\left(s_h^k, a_h^k\right)\right) \prod_{i=0}^m\left(\hat V_{k, h}^{2^i}\left(s_h^k\right)+\left[\mathbb{P} \hat V_{k, h+1}\right]\left(s_h^k, a_h^k\right)^{2^i}\right)   \notag \\
&\leq 2^{m+1} \sum_{k=1}^K \sum_{h=1}^H I_h^k \max \left\{\hat V_{k, h}\left(s_h^k\right)-\left[\mathbb{P} \hat V_{k, h+1}\right]\left(s_h^k, a_h^k\right), 0\right\}   \notag \\
&\leq 2^{m+1} \sum_{k=1}^K \sum_{h=1}^H I_h^k\left[u_{k,h}\left(s_h^k, a_h^k\right)+ 4 \min \left\{1, \beta\left\| \hat \bphi_{k, h, 0}\right\|_{\dot{\hat{\bSigma}}_{k, 0}^{-1}}\right\}\right]   \notag \\
&\leq 2^{m+1}\left(\tilde R_0 + 4 \hat R_0\right),\label{eq:19}
\end{align}
where the first inequality holds due to using $ \EE X^2 \ge (\EE X)^2 $ recursively, the first equality holds due to the fact $x^{2^{m+1}} - y ^{2^{m+1}} = (x-y)\prod_{i=0}^m (x^{2^m} - y^{2^m})$, the second inequality holds due to $\hat V_{k,h}$ belongs to the interval $[0,1]$, the third inequality holds due to Lemma~\ref{lm:ucb}, and the last inequality holds due to $u_{k,h}(s_h^k, a_h^k) = \beta\|\bphi_{V_{h+1}(\cdot;\btheta_k,\pi_k,r_k)}(s_h^k,a_h^k)\|_{\dot{\tilde{\bSigma}}_{k,0}} = \beta\| \tilde\bphi_{k,h,0}\|_{\dot{\tilde{\bSigma}}_{k,0}}$.
If $h_K \le H$, we have $I_{h_k}^k \hat V_{k, h_k+1}^{2^{m+1}}\left(s_{h_k+1}^k\right)\le 1 = 1 - I_H^k$, and if $h_K = H$, $I_{h_k}^k \hat V_{k, h_k+1}^{2^{m+1}}\left(s_{h_k+1}^k\right) = 0 = 1 - I_H^k$, which both give
\begin{align}
\sum_{k=1}^K I_{h_k}^k \hat V_{k, h_k+1}^{2^{m+1}}\left(s_{h_k+1}^k\right) \le \sum_{k=1}^K(1-I_H^k) = G \label{eq:20}
\end{align}
Substituting Equations~\eqref{eq:18}, \eqref{eq:19},\eqref{eq:20} into \eqref{eq:18}, we can get \eqref{eq:inhatS}. For Equation \eqref{eq:tildeS}, similarly, we have
\begin{align}
    &\tilde S_m \le \tilde A_{m+1}+\sum_{k=1}^K \sum_{h=1}^H I_h^k\left[\tilde V_{k, h}^{2^{m+1}}\left(s_h^k\right)-\left(\left[\PP\tilde V_{k, h+1}^{2^m}\right]\left(s_h^k, a_h^k\right)\right)^2\right]+\sum_{k=1}^K I_{h_k}^k \tilde V_{k, h_k+1}^{2^{m+1}}\left(s_{h_k+1}^k\right), \label{eq:S_tilde_medium} \\
    &\sum_{k=1}^K I_{h_k}^k \tilde V_{k, h_k+1}^{2^{m+1}}\left(s_{h_k+1}^k\right) \le \sum_{k=1}^K\left(1-I_H^k\right) = G. \label{eq:S_tilde_second_term}
\end{align}
And we have
\begin{align}
&\sum_{k=1}^K \sum_{h=1}^H I_h^k\left[\hat V_{k, h}^{2^{m+1}}\left(s_h^k\right)-\left(\left[\mathbb{P} \hat V_{k, h+1}^{2^m}\right]\left(s_h^k, a_h^k\right)\right)^2\right]    \notag \\
&\le 2^{m+1} \sum_{k=1}^K \sum_{h=1}^H I_h^k \max \left\{\tilde V_{k, h}\left(s_h^k\right)-\left[\mathbb{P} \tilde V_{k, h+1}\right]\left(s_h^k, a_h^k\right), 0\right\}   \notag \\
&\le 2^{m+1} \sum_{k=1}^K \sum_{h=1}^H I_h^k\left[r_{k,h}(s_h^k,s_h^k) +\min\left\{1,  2\beta\left\| \tilde\bphi_{k,h,0} \right\|_{\dot{\tilde\bSigma}_{k,0}^{-1}}\right\}\right] \notag \\
&\le 2^{m+1} \left(K + 2\tilde R_0\right) \label{eq:S_tilde_third_term}
\end{align}
where the first inequality holds similar to the derivation of \eqref{eq:19}, second inequality follows Lemma~\ref{lm:ucb}, and the third inequality holds due to $\sum_{h=1}^H r_{k,h}(s_h^k,a_h^k) \le 1 $. Plugging Equations~\eqref{eq:S_tilde_second_term} \eqref{eq:S_tilde_third_term} into \ref{eq:S_tilde_medium}, we obtain \ref{eq:intildeS}
\end{proof}

\subsection{Proof of Lemma~\ref{lm:A}}
\begin{proof}[Proof of Lemma~\ref{lm:A}]
The proof follows the proof of Lemma 25 in \citet{zhang2021improved} and Lemma C.7 in \citet{zhou22high}. We use Lemma~\ref{lm:concentration} for $\hat A_m$ and $\tilde A_m$ for each $m$. To avoid confusion, we write $\epsilon,\delta$ in Lemma~\ref{lm:concentration} as $\epsilon',\delta'$. 

Let $ x_{k,h} = I_h^k\left[\left[\PP \hat V_{k,h+1}^{2^m}\right](s_h^k,a_h^k) - \hat V_{k,h+1}^{2^m}(s_{h+1}^k)\right]$, $n=KH$, $\epsilon' = \sqrt{\log(1/\delta')}$, and $\delta' = \delta/(4\log(KH))$. Thus, $\EE\left[\hat x_{k,h}|\cF_{k,h}\right]=0$ and $\EE\left[\hat x_{k,h}^2\big|\cF_{k,h}\right]=I_h^k \left[\VV V_{k,h+1}^{2^m}\right](s_h^k,a_h^k)$. Therefore, for each $m\in \overline{[M]}$, with probability at least $1-\delta$, we have
\begin{align}
|\hat A_m| = \left| \sum_{k=1}^K\sum_{h=1}^H  x_{k,h}\right| \le \sqrt{2\zeta\sum_{k=1}^K\sum_{h=1}^H I_h^k \left[\VV \hat V_{k,h+1}^{2^m}\right](s_h^k,a_h^k)} + \zeta. \notag
\end{align}
Similarly, let $ x_{k,h} = I_h^k\left[\left[\PP \tilde V_{k,h+1}^{2^m}\right](s_h^k,a_h^k) - \tilde V_{k,h+1}^{2^m}(s_{h+1}^k)\right]$, $n=KH$, $\epsilon' = \sqrt{\log(1/\delta')}$, and $\delta' = \delta/(4\log(KH))$. With probability at least $1-\delta$, we have
\begin{align}
    |\tilde A_m| = \left|\sum_{k=1}^K\sum_{h=1}^H  x_{k,h}\right| \le \sqrt{2\zeta\sum_{k=1}^K\sum_{h=1}^H I_h^k \left[\VV \tilde V_{k,h+1}^{2^m}\right](s_h^k,a_h^k)} + \zeta. \notag
\end{align}
Taking union bound over $m\in\overline{[M]}$ completes the proof.
\end{proof}

\subsection{Proof of Lemma~\ref{lm:G}}
\begin{proof}[Proof of Lemma~\ref{lm:G}]
    By the fact that $\det\left(\dot{\hat \bSigma}^{-1/2}_{k+1,m}\right) < \det\left(\hat \bSigma^{-1/2}_{k,H,m}\right)$ and $\det\left(\dot{\tilde \bSigma}^{-1/2}_{k+1,m}\right)\\  < \det\left(\tilde \bSigma^{-1/2}_{k,H,m}\right)$, we have
    \begin{align}
        (1-I_H^k)  &= 1 \Leftrightarrow \exists m \in \overline{[M]},  \det \left(\dot{\hat\bSigma}_{k,  m}^{-1/2}\right) / \det \left(\hat \bSigma_{k, H, m}^{-1/2}\right) > 4~\text{or}~\det \left(\dot{\tilde \bSigma}_{k, m}^{-1/2}\right) / \det \left(\tilde \bSigma_{k, H, m}^{-1/2}\right) > 4  \notag \\
        & \Rightarrow \exists m \in \overline{[M]},  \det \left(\dot{\hat\bSigma}_{k,  m}^{-1/2}\right) / \det \left(\dot{\hat\bSigma}_{k+1,m}^{-1/2}\right) > 4~\text{or}~\det \left(\dot{\tilde \bSigma}_{k, m}^{-1/2}\right) / \det \left(\dot{\tilde \bSigma}_{k+1,m}^{-1/2}\right) > 4
    \end{align}
Let $\hat \cD_m$ and $\tilde \cD_m$ denote the indices $k$ such that 
\begin{align}
    \hat\cD_m := \left\{k\in [K]:~\det \left(\dot{\hat\bSigma}_{k+1,m}\right)/\det \left(\dot{\hat\bSigma}_{k,  m}\right) > 16\right\}  \notag \\
    \tilde\cD_m := \left\{k\in [K]:~\det \left(\dot{\tilde\bSigma}_{k+1,m}\right)/\det \left(\dot{\tilde\bSigma}_{k,  m}\right) > 16\right\} \notag 
\end{align}
Then we have
\begin{align}
    G \le \left| \bigcup_{m=0}^{M-1} \hat\cD_m \cup \bigcup_{m=0}^{M-1} \tilde\cD_m \right| \le \sum_{m=0}^{M-1}\left|\hat\cD_m\right| + \sum_{m=0}^{M-1}|\tilde\cD_m| \notag
\end{align}
For each $m$, we have
\begin{align}
    2\left|\hat \cD_m\right| < \sum_{k \in \hat \cD_m} \log{16} < \sum_{k \in \hat \cD_m} \log \left(\det \left(\dot{\hat\bSigma}_{k+1,m}\right)/\det \left(\dot{\hat\bSigma}_{k,  m}\right)\right) \le \sum_{k=1}^K \log \left(\det \left(\dot{\hat\bSigma}_{k+1,m}\right)/\det \left(\dot{\hat\bSigma}_{k,  m}\right)\right) \notag
\end{align}
Furthermore, since $\det\left(\dot{\hat\bSigma}_{K+1,m}\right) \leq \left(\tr\left(\dot{\hat\bSigma}_{K+1,m}\right)/d \right)^d$ and $\tr\left(\dot{\hat\bSigma}_{K+1,m}\right) \le \tr\left(\lambda I\right) + \\ \sum_{k,h} \left\| \hat\bphi_{k,h.m}  \right\|_2^2 /\hat{\sigma}_{k,h,m}^2 \le d\lambda + KH/\alpha^2 $
\begin{align}
    \sum_{k=1}^K \log \left(\det \left(\dot{\hat\bSigma}_{k+1,m}\right)/\det \left(\dot{\hat\bSigma}_{k,  m}\right)\right) &= \log\left(\det \left(\dot{\hat\bSigma}_{K+1,m}\right)/\det \left(\dot{\hat\bSigma}_{1,  m}\right)\right)  \notag \\
    & \le d\left(\log\left(\lambda + KH/(d\alpha^2)\right) - \log(\lambda)\right) \notag
\end{align}
Therefore $|\hat\cD_m|$ is upper bounded by 
\begin{align}
    |\hat\cD_m| < d/2 \log(1 + KH/(d\lambda\alpha^2)). \notag
\end{align}
And same for $|\tilde\cD_m|$. Taking summation gives the upper bound of $G$.
\end{proof}

\section{Proof of Lower Bound}
Reward-free exploration is more difficult than non-reward-free MDP by definitions since we can easily solve non-reward-free MDP by ignoring its reward and executing reward-free exploration. Thus, we will start with acquiring lower bounds under non-reward-free MDP settings and then obtain sample complexity lower bounds of reward-free exploration. The proof follows ideas of \citet{zhou22high} and \citet{chen2022nearoptimal}.

As noted in Section \ref{se:lower_bound}, we will consider the \textit{hard-to-learn linear mixture MDPs} constructed in \citet{zhou22high}. The state space $\cS = \{x_1, x_2, x_3\}$ and the action space $\cA=\{ \textbf{a}\} = \{-1, 1\}^{d-1}$. The reward function satisfies $r(x_1,\cdot) = r(x_2,\cdot) = 0$, and $r(x_3,\cdot) = \frac{1}{H}$. The transition probability satisfies $\PP(x_2 \mid x_1,\ba) = 1-(\delta + \la\bmu,\ba\ra)$ and $\PP(x_3 \mid x_1,\ba) = \delta + \la\bmu,\ba\ra$, where $\delta = 1/6$ and $\bmu \in \{ -\Delta, \Delta \}^{d-1}$ with $\Delta = \sqrt{\delta/K^{'}} / (4\sqrt{2})$. The transition kernel is formulated as
\begin{align}
& \bphi\left(s' \mid s, \ba\right)= \begin{cases}\left(\alpha(1-\delta),-\beta \ba^{\top}\right)^{\top}, & s=x_1, s'=x_2 ; \\
\left(\alpha \delta, \beta \ba^{\top}\right)^{\top}, & s=x_1, s'=x_3 ; \\
\left(\alpha, \mathbf{0}^{\top}\right)^{\top}, & s \in\left\{x_2, x_3\right\}, s'=s ; \\
\mathbf{0}, & \text { otherwise. }\end{cases} \notag \\
& \btheta =\left(1 / \alpha, \bmu^{\top} / \beta\right)^{\top} . \notag
\end{align}

The following lemma from \citet{zhou22high} lower bounds the regret for linear mixture MDP.
\begin{lemma}\label{lm:mdp-lower}
(Theorem 5.4 in \citet{zhou22high}) Let $B>1$. Then for any algorithm, when $K' \geq \max \left\{3 d^2,(d-1) /(192(B-1))\right\}$, there exists a $B$-bounded linear mixture MDP satisfying Assumptions $3.2$ such that its expected regret $\mathbb{E}[\operatorname{Regret}(K')]$ is lower bounded by $\Omega\left( d \sqrt{K'} /(16 \sqrt{3})\right)$.
\end{lemma}

Given Lemma~\ref{lm:mdp-lower}, we will use the regret lower bound of non-reward-free linear mixture MDPs to derive the sample coomplexity lower bound.

\begin{lemma}\label{lm:nonfree-sample}
Suppose $B>1$. Then for any algorithm $\texttt{ALG}_{NonFree}$ solving non-reward-free linear mixture MDP problems satisfying assumption~\ref{as:reward}, there exist a linear mixture $\cM$ such that $\texttt{ALG}_{NonFree}$ needs to collect at least $\frac{C d^2}{\varepsilon^2}$ episodes to output an $\varepsilon$-policy with probability at least $1-\delta$. Here $C$ is an absolute constant. 
\end{lemma}
\begin{proof}[Proof of Lemma~\ref{lm:nonfree-sample}]
For any algorithm $\texttt{ALG}_{NonFree}$, we construct an algorithm $\texttt{ALG}_{NonFree}^\prime$ executing totally $K_1=cK$ episodes, where c is a constant integer larger than 1. The first $K$ episodes of $\texttt{ALG}_{NonFree}^\prime$ are the same as $\texttt{ALG}_{NonFree}$, and the rest episodes keep executing the policy at the end of episode $K$. By Lemma~\ref{lm:mdp-lower}, we have
\begin{align}
\sum_{k=1}^{K_1} \EE\left[ V(s_1;\btheta^*,\pi^*,r) - V(s_1;\btheta^*,\pi_k,r)\right]  \geq   \frac{c'd\sqrt{K_1}}{16\sqrt{3}}, \label{eq:A1}
\end{align}
for some constant $c'$. In addition, based on the construction of \textit{the hard-to-learn MDPs}, where $K^{'} = K_1$, the per-episode regret is upper bounded by
\begin{align}
\EE\left[V(s_1;\btheta^*,\pi^*,r) - V(s_1;\btheta^*,\pi_k,r)\right] \le \frac{d}{4\sqrt{3K_1}}.   \label{eq:A2}
\end{align}
Thus, calculating \eqref{eq:A1} - $(K_1-K)\times$ \eqref{eq:A2}, and choosing$c=\max\big\{5/c', 2 \big\} $, we have
\begin{align}
\sum_{k=K+1}^{K_1} \EE\left[ V(s_1;\btheta^*,\pi^*,r) - V(s_1;\btheta^*,\pi_k,r)\right] \geq \frac{d\sqrt{K}}{16\sqrt{3c}}. \notag
\end{align}
Since the policies in episode $K+1$ to episode $K_1$ are same to $\pi_K$, we have
\begin{align}
\EE\left[ V(s_1;\btheta^*,\pi^*,r) - V(s_1;\btheta^*,\pi_K,r)\right] \geq \frac{d}{16\sqrt{3cK}c}. \notag
\end{align}
Suppose the $\texttt{ALG}_{NonFree}$ return return a $\varepsilon$-optimal policy with probability $1-\delta$. Then,
\begin{align}
    (1-\delta)\varepsilon + \delta \frac{d}{4\sqrt{3cK}} \geq \frac{d}{16\sqrt{3cK}c}. \notag
\end{align}
Setting $\delta < \min\{1, 1/(4c)\}$, by solving the inequality, we have $ K \geq \frac{C d^2}{\varepsilon^2}$ for some constant $C$.
\end{proof}
Since reward-free MDP is more difficult than non-reward-free MDP, Lemma~\ref{lm:nonfree-sample} directly indicates Theorem~\ref{thm:lowerbound}.
\begin{proof}[Proof of Theorem~\ref{thm:lowerbound}]
We will prove the theorem by contradiction. Assume all reward-free linear mixture MDPs can be solved with sample complexity of $o(\frac{d^2}{\varepsilon^2})$. Then, for any non-reward-free MDP $\cM$, there exists an algorithm $\texttt{ALG}^\prime$ $(\varepsilon,\delta)$ learning its reward-free counterpart $\cM^\prime$ with sample complexity of $o(\frac{d^2}{\varepsilon^2})$. We define $\texttt{ALG}$ solving $\cM$ as follows: it collects $K$ episodes of data and outputs the policy in the same way as $\texttt{ALG}^\prime$ by ignoring the rewards. Then $\texttt{ALG}$ can also $(\varepsilon,\delta)$ learning $\cM$ with sample complexity of $o(\frac{d^2}{\varepsilon^2})$, which contradicts Theorem~\ref{lm:nonfree-sample}.
\end{proof}
Corollary~\ref{cor:lowerboundH} can be viewed as an direct result of Theorem~\ref{thm:lowerbound}.
\begin{proof}[Proof of Corollary~\ref{cor:lowerboundH}]
The hard-to-learn case we consider here is basically same as we consider in Theorem~\ref{thm:lowerbound}, except replacing reward function with $r(x_1,\cdot) = r(x_2,\cdot) = 0$, and $r(x_3,\cdot) = 1$. Since the reward here is $H$ times the reward in Theorem~\ref{thm:lowerbound}, the suboptimality is also $H$ times. Therefore, $\varepsilon/H$-optimal policy in Theorem~\ref{thm:lowerbound} is a $\varepsilon$-optimal policy here. According to Theorem~\ref{thm:lowerbound}, the sample complexity required to achieve such a policy with probability at least $1-\delta$ is $\Omega(\frac{H^2d^2}{\varepsilon^{2}})$.

\end{proof}

\section{Auxiliary Lemmas}

\begin{lemma}\label{lm:det}
Suppose $\Ab,\Bb \in \RR^{d\times d}$ are two positive definite matrices satisfying $\bA \succeq \Bb$, then for any $\bx \in \RR^d$, $ \| x \|_{\Ab} \le \| \xb \|_{\Bb} \sqrt{\det(\Ab)/\det(\Bb)} $
\end{lemma}

\begin{lemma}[Lemma 11, \citealt{zhang2021model}]\label{lm:concentration} Let $M>0$ be a constant. Let $\left\{x_i\right\}_{i=1}^n$ be a stochastic process, $\mathcal{G}_i=\sigma\left(x_1, \ldots, x_i\right)$ be the $\sigma$-algebra of $x_1, \ldots, x_i$. Suppose $\mathbb{E}\left[x_i \mid \mathcal{G}_{i-1}\right]=0,\left|x_i\right| \leq M$ and $\mathbb{E}\left[x_i^2 \mid \mathcal{G}_{i-1}\right]<\infty$ almost surely. Then, for any $\delta, \varepsilon>0$, we have
\begin{align}
&\mathbb{P}\left(\left|\sum_{i=1}^n x_i\right| \leq 2 \sqrt{2 \log (1 / \delta) \sum_{i=1}^n \mathbb{E}\left[x_i^2 \mid \mathcal{G}_{i-1}\right]}+2 \sqrt{\log (1 / \delta)} \varepsilon+2 M \log (1 / \delta)\right)  \notag \\
&\quad>1-2\left(\log \left(M^2 n / \varepsilon^2\right)+1\right) \delta
\end{align}
\end{lemma}

\begin{lemma}\label{lm:sequence}
(Lemma 12 in \citet{zhang2021improved}) Let $\lambda_1, \lambda_2, \lambda_4>0, \lambda_3 \geq 1$ and $\kappa=\max \left\{\log _2 \lambda_1, 1\right\}$. Let $a_1, \ldots, a_\kappa$ be non-negative real numbers such that $a_i \leq \min \left\{\lambda_1, \lambda_2 \sqrt{a_i+a_{i+1}+2^{i+1} \lambda_3}+\lambda_4\right\}$ for any $1 \leq i \leq \kappa$. Let $a_{\kappa+1}=\lambda_1$. Then we have $a_1 \leq 22 \lambda_2^2+6 \lambda_4+4 \lambda_2 \sqrt{2 \lambda_3}$.
\end{lemma}

\end{document}